\documentclass[runningheads]{llncs}

\usepackage{eccv}

\usepackage{eccvabbrv}

\usepackage{graphicx}
\usepackage{booktabs}
\usepackage{tikz}
\usetikzlibrary{decorations.pathreplacing}

\usepackage[accsupp]{axessibility}  

\usepackage{hyperref}

\usepackage{orcidlink}

\usepackage{multirow}
\usepackage{amsmath} 
\DeclareMathOperator*{\argmax}{arg\,max}
\usepackage{algorithm}
\usepackage{algpseudocode} 
\usepackage{amssymb} 

\newif\ifincludesuppl
\includesuppltrue

\newcommand{\ofmainpaper}{}

\begin{document}

\title{Cross-View Yaw Estimation in Location Uncertainty with Line-Aligning Yaw Scoring}

\titlerunning{LAYS: Line-Aligning Yaw Scoring}

\def\ourmethod{LAYS}
\def\ourmethodfullname{Line-Aligning Yaw Scoring}

\let\citep\relax
\let\citet\relax

\NewDocumentCommand{\citep}{ o m }{%
  \IfNoValueTF{#1}
    {\cite{#2}}
    {\cite[#1]{#2}}%
}

\NewDocumentCommand{\citet}{ o m }{%
  \IfNoValueTF{#1}
    {\cite{#2}}
    {\cite[#1]{#2}}%
}

\renewcommand{\thefootnote}{\fnsymbol{footnote}}
\author{Taeho Kang
\inst{1}\orcidlink{0000-0002-4556-5588} \and
Nairan Zhang\thanks{This work was done while working at Amazon.}\inst{2}\orcidlink{0009-0009-2699-1598}
\and 
Yelin Kim\inst{3}\orcidlink{0000-0002-6503-4637} \and
Yujiao Shi\inst{4}\orcidlink{0000-0001-6028-9051} \and
Youngki Lee\inst{1}\orcidlink{0000-0002-1319-7071}}

\authorrunning{Kang et al.}

\institute{Seoul National University, South Korea \and
Meta, USA \and
Amazon, USA \and
ShanghaiTech University, China\\
\email{taeho.kang@hcs.snu.ac.kr, nairanzhang@meta.com, kimyelin@amazon.com, shiyj2@shanghaitech.edu.cn, youngkilee@snu.ac.kr}}

\maketitle
\renewcommand{\thefootnote}{\arabic{footnote}}
\setcounter{footnote}{0}

\begin{abstract}
Accurate yaw estimation is a bottleneck in cross-view localization between ground view and Bird's Eye View (BEV). Existing methods couple yaw with translation and rely on height or projection assumptions that degrade under large yaw ambiguity. We disentangle yaw from location accuracy and introduce LAYS, a radially invariant line-consensus voting method. By exploiting the radial invariance of our formulation, we achieve sub-degree yaw precision via 3D voting over all candidate poses, while eliminating the need for accurate location. Our key observation is that a ground-image column matched to BEV pixels induces the same yaw across all camera positions along the radial direction of the pixels. LAYS matches BEV pixels to ground columns using feature similarity and accumulates the induced yaw votes into discrete 3D bins, where correct correspondences along the radial line concentrate into a sharp peak for the correct yaw. Experiments on Mapillary, Ford, KITTI, and VIGOR show significant gains under unknown yaw, particularly for normal FoV with unknown yaw (+28$\sim$45\%p), and using LAYS as a yaw prior improves downstream 3-DoF localization.
  \keywords{cross-view localization \and aerial view \and 3D vision}
\end{abstract}
\section{Introduction}
\label{sec:intro}
Accurate global localization is essential for autonomous navigation, robotics, and AR/MR systems.
In these applications, orientation errors are particularly harmful: even small angular deviations can lead to severe navigation drift or visual misalignment. Among the three rotation components, yaw is uniquely challenging. Unlike roll and pitch, which can be inferred from gravity-aligned cues, yaw lacks a direct geometric reference in ground images. As a result, yaw becomes a persistent bottleneck. Without accurate yaw, global orientation alignment degenerates into an unstable joint optimization of translation and rotation.

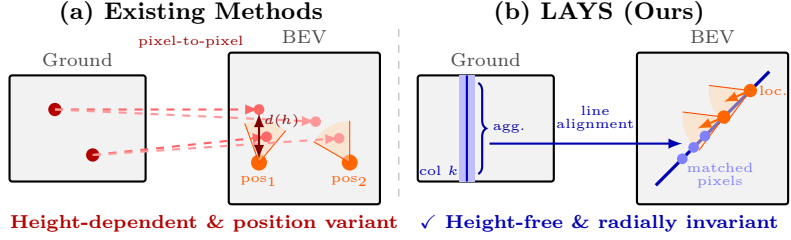
\begin{figure}[!ht]
\centering
\setlength{\abovecaptionskip}{0pt}
\setlength{\belowcaptionskip}{-2pt}
\begin{tikzpicture}[
    >=latex,
    font=\small,
    imgbox/.style={draw, thick, minimum width=1.8cm, minimum height=1.4cm, rounded corners=1pt},
    bevbox/.style={draw, thick, minimum width=2.0cm, minimum height=2.0cm, rounded corners=1pt},
    lbl/.style={font=\footnotesize\bfseries, anchor=south},
    sublbl/.style={font=\scriptsize, text=black!70},
]

\begin{scope}[shift={(0,0)}]
    \node[lbl] at (2.4, 2.15) {(a) Existing Methods};

    \node[imgbox, fill=black!5] (grd_a) at (0.9, 0.9) {};
    \node[sublbl, anchor=south] at (grd_a.north) {Ground};

    \fill[red!70!black] (0.6, 1.15) circle (2.5pt);
    \fill[red!70!black] (1.1, 0.55) circle (2.5pt);

    \node[bevbox, fill=black!5] (bev_a) at (3.9, 0.9) {};
    \node[sublbl, anchor=south] at (bev_a.north) {BEV};

    \fill[orange!80!red] (3.3, 0.45) circle (3pt);
    \fill[orange!30, opacity=0.5] (3.3, 0.45) -- ++(50:0.55) arc (50:110:0.55) -- cycle;
    \draw[orange!60!red, thin] (3.3, 0.45) -- ++(50:0.55);
    \draw[orange!60!red, thin] (3.3, 0.45) -- ++(110:0.55);
    \node[font=\tiny, orange!70!red, anchor=north] at (3.3, 0.38) {pos$_1$};

    \fill[orange!80!red] (4.5, 0.45) circle (3pt);
    \fill[orange!30, opacity=0.5] (4.5, 0.45) -- ++(90:0.55) arc (90:150:0.55) -- cycle;
    \draw[orange!60!red, thin] (4.5, 0.45) -- ++(90:0.55);
    \draw[orange!60!red, thin] (4.5, 0.45) -- ++(150:0.55);
    \node[font=\tiny, orange!70!red, anchor=north] at (4.5, 0.38) {pos$_2$};

    \fill[red!60] (3.30, 1.15) circle (2pt);
    \draw[red!60, dashed, thick, ->] (0.6, 1.15) -- (3.30, 1.15);
    \fill[red!40] (4.05, 0.99) circle (2pt);
    \draw[red!40, dashed, thick, ->] (0.6, 1.15) -- (4.05, 0.99);

    \fill[red!60] (3.40, 0.79) circle (2pt);
    \draw[red!60, dashed, thick, ->] (1.1, 0.55) -- (3.40, 0.79);
    \fill[red!40] (4.36, 0.77) circle (2pt);
    \draw[red!40, dashed, thick, ->] (1.1, 0.55) -- (4.36, 0.77);

    \draw[red!50!black, thick, <->] (3.3, 0.50) -- (3.3, 1.10);
    \node[font=\tiny, red!50!black, anchor=east] at (3.95, 1.0) {$d(h)$};

    \node[font=\tiny, red!60!black, anchor=south] at (2.4, 1.85) {pixel-to-pixel};

    \node[font=\scriptsize\bfseries, red!70!black, align=center] at (2.4, -0.35)
        {\texttimes~Height-dependent \& position variant};
\end{scope}

\draw[black!30, dashed] (5.15, 0.0) -- (5.15, 2.3);

\begin{scope}[shift={(5.4, 0)}]
    \node[lbl] at (2.4, 2.15) {(b) LAYS (Ours)};

    \node[imgbox, fill=black!5] (grd_b) at (0.9, 0.9) {};
    \node[sublbl, anchor=south] at (grd_b.north) {Ground};

    \fill[blue!25] (0.55, 0.2) rectangle (0.75, 1.6);
    \draw[blue!70!black, thick] (0.65, 0.2) -- (0.65, 1.6);
    \node[font=\tiny, blue!70!black, anchor=east] at (0.65, 0.35) {col $k$};

    \draw[blue!60!black, thick, decorate, decoration={brace, amplitude=3pt, mirror}]
        (0.8, 0.3) -- (0.8, 1.5) node[midway, right, font=\tiny, xshift=2pt] {agg.};

    \node[bevbox, fill=black!5] (bev_b) at (3.9, 0.9) {};
    \node[sublbl, anchor=south] at (bev_b.north) {BEV};

    \draw[blue!70!black, very thick] (3.15, 0.15) -- (4.65, 1.65);
    
    \fill[orange!30, opacity=0.5] (4.05, 1.05) -- ++(173:0.55) arc (173:233:0.55) -- cycle;
    \draw[orange!60!red, thin] (4.05, 1.05) -- ++(173:0.55);
    \draw[orange!60!red, thin] (4.05, 1.05) -- ++(233:0.55);

    \fill[orange!30, opacity=0.5] (4.4, 1.4) -- ++(173:0.55) arc (173:233:0.55) -- cycle;
    \draw[orange!60!red, thin] (4.4, 1.4) -- ++(173:0.55);
    \draw[orange!60!red, thin] (4.4, 1.4) -- ++(233:0.55);

    \fill[orange!80!red] (4.05, 1.05) circle (2.5pt);
    \fill[orange!80!red] (4.4, 1.4) circle (2.5pt);
    \node[font=\tiny, orange!70!red, anchor=west] at (4.4, 1.4) {loc.};

    \fill[blue!50] (3.5, 0.5) circle (2.0pt);
    \fill[blue!50] (3.65, 0.65) circle (2.0pt);
    \fill[blue!50] (3.8, 0.8) circle (2.0pt);

    \draw[orange!80!red, thick, ->] (4.05, 1.05) -- (3.7, 0.9);
    \draw[orange!80!red, thick, ->] (4.4, 1.4) -- (4.05, 1.25);

    \draw[blue!70!black, thick, ->] (0.95, 0.7) -- (3.5, 0.7);
    \node[font=\tiny, blue!70!black, anchor=south, align=center] at (2.35, 0.75) {line\\[-1pt]alignment};
    
    \node[font=\tiny, blue!50, anchor=south, align=center] at (4.0, 0.0) {matched\\[-1pt]pixels};

    \node[font=\scriptsize\bfseries, blue!70!black, align=center] at (2.4, -0.35)
        {\checkmark~Height-free \& radially invariant};
\end{scope}

\end{tikzpicture}
\caption{ 
\textbf{(a)}~Existing methods rely on pixel-to-pixel correspondences. Their BEV projections shift based on the assumed camera pose entangling location and yaw, and require a distance estimate $d(h)$ dependent on ground height $h$. \textbf{(b)}~\ourmethod~aggregates ground pixels vertically into a column and aligns it to a BEV radial direction. This height-free, column-to-line correspondence makes estimated yaw invariant for any candidate position along the radial line, successfully decoupling yaw from location.}
\label{fig:comparison_baseline_ours}
\end{figure}

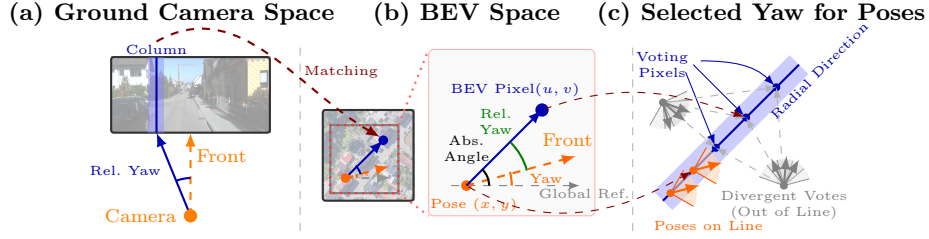
\begin{figure*}[!ht]
\centering
\setlength{\abovecaptionskip}{0pt}
\setlength{\belowcaptionskip}{-2pt}
\begin{tikzpicture}[
    >=latex,
    font=\small,
    lbl/.style={font=\footnotesize\bfseries, anchor=south},
]

\begin{scope}[shift={(0,0)}]
    \node[lbl] at (1, 2.6) {(a) Ground Camera Space};
    
    \node[inner sep=0pt, anchor=south west, opacity=0.4] at (0.2, 1.3) {\includegraphics[width=2.1cm, height=1.0cm]{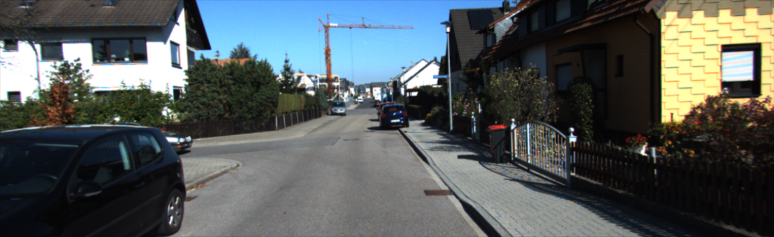}};
    \draw[thick, black!80, rounded corners=1pt] (0.2, 1.3) rectangle (2.3, 2.3);
    
    \fill[blue!50, opacity=0.4] (0.7, 1.3) rectangle (0.9, 2.3);
    \draw[blue!70!black, thick] (0.8, 1.3) -- (0.8, 2.3);
    \node[font=\tiny, blue!70!black, anchor=south, inner sep=1pt] (gc_node) at (0.8, 2.3) {Column};
    
    \draw[orange, dashed, thick, ->] (1.25, 0.2) -- (1.25, 1.3);
    \node[font=\scriptsize, orange, anchor=west, inner sep=2pt] at (1.25, 1.0) {Front};
    
    \draw[blue!70!black, thick, ->] (1.25, 0.2) -- (0.8, 1.3);
    
    \draw[blue!70!black, thick, -] (1.25, 0.7) arc (90:112.5:0.5);
    \node[font=\tiny, blue!70!black, anchor=east, inner sep=1pt] at (0.9, 0.8) {Rel. Yaw};
    
    \fill[orange] (1.25, 0.2) circle (2.5pt) node[left=2pt, font=\scriptsize] {Camera};
\end{scope}

\draw[black!30, dashed] (2.7, 0.0) -- (2.7, 2.5);

\begin{scope}[shift={(3.0,0)}]
    \node[lbl] at (1.9, 2.6) {(b) BEV Space};
    
    \begin{scope}[yshift=0.2cm]
        \node[inner sep=0pt, anchor=south west, opacity=0.4] at (0, 0.2) {\includegraphics[width=1.2cm, height=1.2cm]{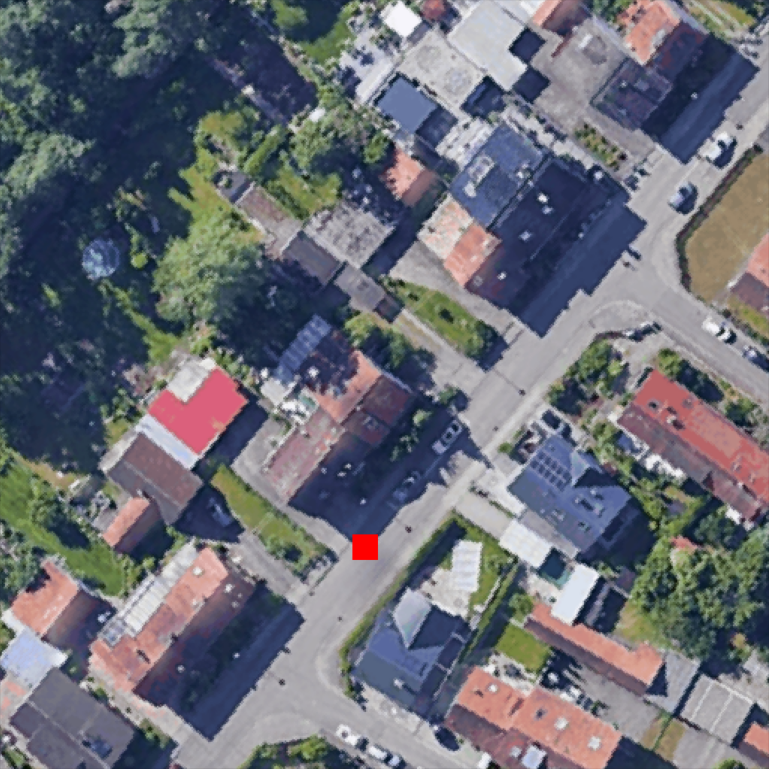}};
        \draw[thick, black!80, rounded corners=1pt] (0, 0.2) rectangle (1.2, 1.4);
        
        \coordinate (mini_pose) at (0.3, 0.5);
        \coordinate (mini_pixel) at (0.8, 1.0);
        
        \draw[gray!80, thick, dashed, ->] (mini_pose) -- ++(0:0.6); 
        \draw[orange!80!red, thick, dashed, ->] (mini_pose) -- ++(15:0.6); 
        \draw[blue!70!black, thick, ->] (mini_pose) -- (mini_pixel); 
        
        \draw[orange!80!red, thick] (0.45, 0.5) arc (0:15:0.2);
        \draw[blue!70!black, thick] (0.5, 0.55) arc (15:45:0.25);
        
        \fill[orange!80!red] (mini_pose) circle (1.8pt);
        \fill[blue!70!black] (mini_pixel) circle (1.8pt) node[inner sep=1pt] (real_pixel) {};
        
        \coordinate (z_bl) at (0.1, 0.3);
        \coordinate (z_tr) at (1.0, 1.2);
        \draw[gray, thick] (z_bl) rectangle (z_tr);
        \draw[red!80, thick, dotted] (z_bl) rectangle (z_tr);
        
        \draw[red!60, thick, dotted] (1.0, 1.2) -- (1.4, 2.2);
        \draw[red!60, thick, dotted] (1.0, 0.3) -- (1.4, 0.0);

        \begin{scope}[shift={(1.5, 0)}]
            \draw[red!30, fill=black!2, rounded corners=2pt] (-0.1, 0.0) rectangle (2.1, 2.2);
        
            \draw[gray, dashed, ->] (0.2, 0.4) -- (1.9, 0.4) node[below left, font=\tiny, xshift=20pt, yshift=0pt, inner sep=1pt] {Global Ref.};
            
            \coordinate (pose) at (0.4, 0.4);
            \fill[orange!80!red] (pose) circle (2pt) node[inner sep=0pt] (b_pose_node) {};
            \node[font=\tiny, orange!80!red, anchor=north] at (0.5, 0.35) {Pose $(x,y)$};
            
            \draw[orange!80!red, thick, dashed, ->] (pose) -- ++(15:1.5) node[text=orange, font=\scriptsize, anchor=north west, xshift=-14pt, yshift=10pt, inner sep=1pt] {Front};
            
            \coordinate (pixel) at (1.4, 1.4);
            \fill[blue!70!black] (pixel) circle (2.5pt) node[inner sep=0pt] (b_pixel_node) {};
            \node[font=\tiny, blue!70!black, anchor=south, xshift=-10pt, yshift=2pt, inner sep=3pt] at (pixel) {BEV Pixel$(u,v)$};
            
            \draw[blue!70!black, thick, ->] (pose) -- (pixel);
            
            \draw[orange!80!red, thick, -] (1.0, 0.4) arc (0:15:0.7);
            \node[font=\tiny, orange!80!red, anchor=west] at (1.12, 0.5) {Yaw};
            
            \draw[green!50!black, thick, -] (1.2, 0.6) arc (15:45:0.82);
            \node[font=\tiny, green!50!black, anchor=east, align=center, inner sep=1pt] at (1.05, 1.25) {Rel.\\Yaw};
            
            \draw[black, thick, -] (0.7, 0.4) arc (0:45:0.3);
            \node[font=\tiny, black, anchor=south, align=center, inner sep=1pt] at (0.4, 0.65) {Abs.\\Angle};
        \end{scope}
    \end{scope}
\end{scope}

\draw[red!50!black, thick, dashed, ->, out=30, in=150, overlay] (gc_node.north) to node[below right, font=\tiny, align=center, xshift=10pt, yshift=-1pt] {Matching} (real_pixel.north west);

\draw[black!30, dashed] (7.1, 0.0) -- (7.1, 2.5);

\begin{scope}[shift={(7.3,0)}] 
    \node[lbl] at (1.5, 2.6) {(c) Selected Yaw for Poses};
    
    \begin{scope}[yshift=0.2cm]
        \draw[line width=10pt, blue!20] (0.2, 0.2) -- (2.0, 2.0);
        \draw[thick, blue!70!black] (0.2, 0.2) -- (2.0, 2.0);
        
        \node[font=\tiny, blue!70!black, rotate=45, anchor=north west] at (1.45, 1.35) {Radial Direction};
        
        \coordinate (p1) at (0.9, 0.9);
        \coordinate (p2) at (1.3, 1.3);
        \coordinate (p3) at (1.7, 1.7);
        \fill[blue!70!black] (p1) circle (1.5pt);
        \fill[blue!70!black] (p2) circle (1.5pt) node[inner sep=0pt] (c_pixel_node) {};
        \fill[blue!70!black] (p3) circle (1.5pt);
        
        \node[font=\tiny, blue!70!black, anchor=east, align=center] at (0.65, 2.0) {Voting\\Pixels};
        \draw[blue!70!black, thin, ->, shorten >=1pt] (0.5, 1.9) -- (p1);
        \draw[blue!70!black, thin, ->, shorten >=1pt] (0.6, 2.0) -- (p2);
        \draw[blue!70!black, thin, ->, shorten >=1pt] (0.7, 2.05) to[out=0, in=135] (p3);
        
        \coordinate (pose1) at (0.3, 0.3);
        \fill[orange!30, opacity=0.6] (pose1) -- ++(-30:0.35) arc (-30:60:0.35) -- cycle; 
        \draw[orange!60!red, thin] (pose1) -- ++(-30:0.35);
        \draw[orange!60!red, thin] (pose1) -- ++(60:0.35);
        \fill[orange!80!red] (pose1) circle (1.5pt);
        \draw[orange!80!red, thick, ->] (pose1) -- ++(15:0.4);
        
        \coordinate (pose2) at (0.6, 0.6);
        \fill[orange!30, opacity=0.6] (pose2) -- ++(-30:0.35) arc (-30:60:0.35) -- cycle;
        \draw[orange!60!red, thin] (pose2) -- ++(-30:0.35);
        \draw[orange!60!red, thin] (pose2) -- ++(60:0.35);
        \fill[orange!80!red] (pose2) circle (1.5pt) node[inner sep=0pt] (c_pose_node) {};
        \draw[orange!80!red, thick, ->] (pose2) -- ++(15:0.4);
        
        \node[font=\tiny, orange!80!red, anchor=north] at (0.8, 0.05) {Poses on Line};
        
        \coordinate (bad1) at (0.2, 1.5);
        \draw[black!30, dashed, ->] (bad1) -- (p1);
        \draw[black!30, dashed, ->] (bad1) -- (p2);
        \draw[black!30, dashed, ->] (bad1) -- (p3);
        
        \foreach \ang in {-75, -40, -20} {
            \fill[gray!30, opacity=0.3] (bad1) -- ++(\ang-45:0.35) arc (\ang-45:\ang+45:0.35) -- cycle;
            \draw[gray!80, thin] (bad1) -- ++(\ang-45:0.35);
            \draw[gray!80, thin] (bad1) -- ++(\ang+45:0.35);
        }
        \fill[gray] (bad1) circle (1.5pt);
        \draw[gray, thick, ->] (bad1) -- ++(-75:0.4); 
        \draw[gray, thick, ->] (bad1) -- ++(-40:0.4); 
        \draw[gray, thick, ->] (bad1) -- ++(-20:0.4); 
        
        \coordinate (bad2) at (1.8, 0.4);
        \draw[black!30, dashed, ->] (bad2) -- (p1);
        \draw[black!30, dashed, ->] (bad2) -- (p2);
        \draw[black!30, dashed, ->] (bad2) -- (p3);
        
        \foreach \ang in {115, 85, 60} {
            \fill[gray!30, opacity=0.3] (bad2) -- ++(\ang-45:0.35) arc (\ang-45:\ang+45:0.35) -- cycle;
            \draw[gray!80, thin] (bad2) -- ++(\ang-45:0.35);
            \draw[gray!80, thin] (bad2) -- ++(\ang+45:0.35);
        }
        \fill[gray] (bad2) circle (1.5pt);
        \draw[gray, thick, ->] (bad2) -- ++(115:0.4); 
        \draw[gray, thick, ->] (bad2) -- ++(85:0.4); 
        \draw[gray, thick, ->] (bad2) -- ++(60:0.4); 
        
        \node[font=\tiny, gray, align=center, anchor=north] at (1.8, 0.45) {Divergent Votes\\(Out of Line)};
    \end{scope}
\end{scope}

\draw[red!50!black, dashed, ->, out=20, in=160, overlay] (b_pixel_node) to (c_pixel_node);

\draw[red!50!black, dashed, ->, out=-30, in=210, overlay] (b_pose_node) to (c_pose_node);

\end{tikzpicture}
\caption{\textbf{Yaw Voting Mechanism.} \textbf{(a)} A ground column defines a specific relative yaw. \textbf{(b)} For any candidate pose, the vector toward its matched BEV pixel defines an absolute angle. The resulting true yaw is calculated by subtracting the relative yaw from this absolute angle. \textbf{(c)} A matched BEV pixel casts yaw votes for candidate poses. Poses geometrically aligned on the correct radial direction consistently vote for the exact same true yaw (orange arrows), while poses out of line compute divergent, inconsistent yaws across different pixel matches (gray arrows).}
\label{fig:angle_est}
\end{figure*}
\paragraph{Problem Setup.}
We consider the following scenario: given a ground-level image and a Bird’s Eye View (BEV) image, the 2D camera location has noise (e.g., on the order of $\pm20\,\mathrm{m}$), while the yaw angle is unknown (up to $\pm180^\circ$). Our objective is to estimate yaw. 
Unlike full 3-DoF cross-view localization~\citep{shi2019optimal}, which jointly estimates 2D location and yaw, our method decouples yaw robustness from location accuracy through Proposition~\ref{prop:radial}'s radial invariance: enabling sub-degree yaw precision without accurate camera position.

Although cross-view localization has advanced significantly, existing methods primarily focus on estimating translation, while yaw is assumed to be approximately known or perturbed only slightly (e.g., within $\pm10^\circ$). Moreover, many approaches tightly couple translation and rotation, relying on point-to-point correspondences with height-dependent projections~\citep{shi_2022CVPR_beyond, Lentsch_2023SliceMatch, Wang_2024OVCL} (Fig.~\ref{fig:comparison_baseline_ours}-a) that break down on slopes or in urban clutter.

We argue that yaw estimation should instead be isolated as an independent subproblem. If the camera position were perfectly known, a single-point correspondence would determine the yaw. However, under position uncertainty, a stronger geometric primitive is required to determine the yaw. To this end, we introduce \textbf{\ourmethod (\ourmethodfullname)}, which formulates cross-view yaw estimation as a \emph{line alignment problem} rather than point-level matching.

\ourmethod~achieves this by vertically aggregating ground image pixels into column features. This inherently removes any dependence on ground height or camera elevation, gracefully bypassing the ground height assumptions of prior work. By matching one such ground column to a set of BEV pixels, we define a radial line (Fig.~\ref{fig:comparison_baseline_ours}-b). This column-to-line alignment is \emph{radially invariant}: all candidate camera positions along this ray share the exact same absolute direction to the BEV pixel. This yields a consistent yaw estimate regardless of location uncertainty along that line (formalized in Proposition~\ref{prop:radial}).

This geometric invariance motivates a robust voting strategy (Fig.~\ref{fig:angle_est}): for each candidate position, the match score between a BEV pixel and a ground column votes for a yaw bin determined by their absolute-to-relative angle geometry. True matches accumulate a strong consensus for the correct yaw, while incorrect matches disperse.
Unlike traditional voting paradigms that accumulate votes directly for geometric elements, our approach votes for yaw, computed from the geometric relationship between matched features, at all 2D positions. We construct a 3D voting tensor (yaw, 2D position) over all candidate poses, where Proposition~\ref{prop:radial} ensures correct matches concentrate at the true yaw along the radial line, enabling yaw recovery without pinpointing position.

Decoupling of yaw estimation from 3-DoF localization has two implications. Importantly, \textbf{yaw robustness is decoupled from location accuracy}: Proposition~\ref{prop:radial} guarantees that location uncertainty is handled through radial voting consensus, enabling sub-degree yaw precision without knowing position. Second, precise yaw estimation serves as a reusable module that enhances downstream 3-DoF pipelines beyond what joint optimization alone achieves.
Experiments on Mapillary Geo-Localization~\citep{sarlin2023orienternet}, Ford Multi-AV~\citep{Ford-Multi-AV}, KITTI~\citep{Geiger2013KITTI}, and VIGOR~\citep{zhu2021vigor} confirm the effectiveness of \ourmethod, consistently outperforming state-of-the-art baselines by large margins.

\paragraph{Contributions.}  
This paper makes the following contributions:  
\begin{itemize}
    \item We \textbf{formulate} \ourmethod, achieving sub-degree yaw precision under $\pm$20m location uncertainty and $\pm$180$^\circ$ yaw ambiguity, without flat ground or camera height assumptions.
    \item We \textbf{introduce} a line alignment framework integrating column-wise feature extraction, ground-BEV matching, and pairwise yaw voting, decoupling yaw from location through geometric consensus.
    \item We \textbf{demonstrate} state-of-the-art improvements on four datasets, establishing yaw as a tractable subproblem in global pose estimation that further boosts downstream localization.
\end{itemize}
\section{Related Works}
\label{sec:related_works}
\paragraph{Vision-based Global Positioning System}
In early days~\citep{geoloc_survey}, methods localized the image coarsely with scale~\citep{Weyand2016PlaNetP, Hays:2008:im2gps}. Advanced localization uses feature databases~\citep{2022navvis} (e.g., Google Map's Visual Positioning System~\citep{2024vpsreview}) or city-scale Structure from Motion~\citep{buildingrome, worldwidepose} with SLAM~\citep{scalable6-dof}. These provide accurate localization but require costly, large-scale feature databases. Alternatively, publicly available 2D maps are utilized~\citep{sarlin2023orienternet, 2024maplocnet}, with multi-modal learning to construct a semantic map~\cite{sarlin2023snap}. 2D maps provide rich semantics, such as roads and buildings, but they lack detailed visual cues for localization.

\paragraph{Visual Camera Rotation Estimation}
Many rotation-focused pose estimation methods are developed for applications such as Unmanned Aerial Vehicle operations~\citep{liu2022visual}, where orientation is critical. Rotation is often part of 6-DoF pose estimation, as location and rotation are typically inseparable in localization. Some focus on frame-to-frame rotation using omnidirectional cameras~\citep{mlpcompass2024} or handheld cameras in crowded scenes~\citep{2023iccvcamrot}. For 2D rotation, geometry-aware methods predict the horizon line from a single image~\citep{Xian2019UprightNetGC}, and roll and pitch estimation using camera parameters has been proposed~\citep{jin2023perspective}. Unlike yaw, local features can estimate roll and pitch by using surrounding structural cues to align images with the direction of gravity. In a cross-view setup, rotation estimation is primarily in 2D by using gravity-aligned images.

\paragraph{Ground and Aerial Cross-View Localization}
Cross-view localization matches a ground-level query to geo-referenced aerial data by bridging the viewpoint gap. Early methods retrieved locations from sparse candidates~\citep{NIPS2019_shi, Liu2019lending, shi2019optimal, 2022softgeo, Shi_2022_ACCV_CVLNet, zhu2022transgeo}. Orientation estimation was introduced by applying a polar transform to the satellite image centered at a known position~\citep{2020shi}, and by matching non-aligned images~\citep{zhu2021vigor}. Evolved methods estimate precise positions~\citep{xia2022visual} and yaw in BEV by matching deep features~\citep{shi_2022CVPR_beyond}, often assuming a perspective transform at a specific ground height~\citep{Fervers_2023_CVPR_CSLA, Song_2023_LearningFlow, Wang_2023PureACL, Wang_2024OVCL}. Techniques include horizontally splitting images~\citep{Lentsch_2023SliceMatch}, rolling descriptors~\citep{xia2023_CCVPE}, refining homography between BEV and projected ground images~\citep{Song_2023_LearningFlow, NEURIPS2023_homography}, and iterative pose refinement~\citep{Wang_2023PureACL, Wang_2024OVCL}. OrienterNet~\citep{sarlin2023orienternet} projects ground features onto a 2D map and scores each 3-DoF pose via cross-correlation, matching \emph{what is at a projected location}; this relies on a projection model for ground-to-BEV mapping. Recent method~\citet{Xia_2025_CVPR_FG2} maps the features of the ground view to a 3D grid and uses Procrustes analysis to estimate the pose. A line of works~\citet{shi_2022CVPR_beyond, Shi_2023Boosting, g2sweakly_2024_shi} estimate yaw from joint regression or 3-DoF optimization, focusing on slight yaw noise in driving scenes. We propose a yaw estimate that matches \emph{the direction a feature lies in} rather than the projected location, without requiring known position or ground height for ground-to-BEV mapping.
\begin{figure}[!ht]
\setlength{\abovecaptionskip}{0pt}
    \setlength{\belowcaptionskip}{0pt}
    \includegraphics[width=\linewidth]{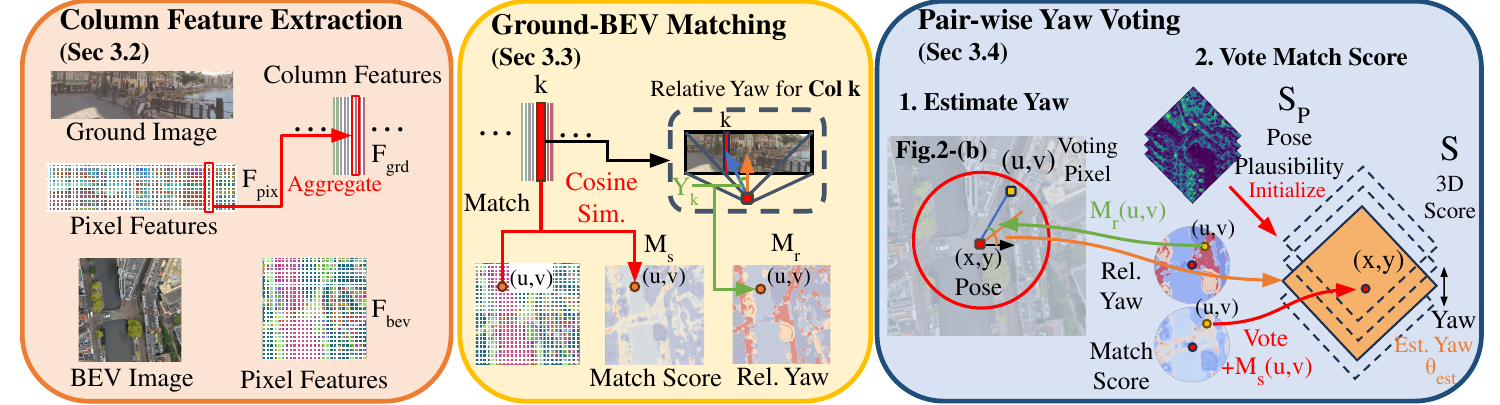}
    \caption{Overview of the \ourmethod~framework. \textbf{(a)} U-Net encoders process ground and BEV images. Ground pixel features are aggregated column-wise to produce column features $\mathbf{F}_\text{grd}$. \textbf{(b)} Each BEV pixel is matched to the most similar ground column, producing a match score $\mathbf{M_s}$ and relative yaw $\mathbf{M_r}$. \textbf{(c)} For each pair of 2D pose and BEV pixel, match scores are voted into yaw bins based on the geometric relationship between absolute and relative angles. (as in Fig.~\ref{fig:angle_est}-b) The output yaw is the yaw from the highest scoring bin $\theta~\text{from}~S(\theta, (x, y))$.}
    \label{fig:architecture}
\end{figure}

\section{Method}
\subsection{Overview}
\label{sec:overview}

\ourmethod~framework is illustrated in Fig.~\ref{fig:architecture}. We follow standard cross-view localization input. \ourmethod~takes as input an undistorted ground-level image \( \text{I}_{grd} \in \mathbb{R}^{3 \times H_g \times W_g} \), nearly gravity-aligned along the $y$ axis, and a BEV image \( \text{I}_{bev} \in \mathbb{R}^{3 \times H_b \times W_b} \), a top-down view that spatially covers the surrounding area. The camera intrinsics are given. Importantly, \textbf{our formulation does not assume ground-truth location; the 2D pose $(x, y)$ remains unknown and is exhaustively tested} over all candidate positions within the noise range. The goal is to estimate yaw, formulated to output a 2D position-conditioned discrete yaw score bins $S(\theta, (x, y))$, a 3D array $S \in \mathbb{R}^{|\Theta| \times |\mathcal{X}| \times |\mathcal{Y}|}$, where \( (x, y) \) represents the discrete 2D position in BEV coordinates, and \( \theta \) denotes the yaw. The yaw from the highest scoring 3D bin is the estimated yaw.

Our relative yaw-based formulation is visualized in Fig.~\ref{fig:angle_est}. Each ground column defines a relative yaw from the camera's front (Fig.~\ref{fig:angle_est}-a). For any candidate 2D pose, the yaw is estimated by subtracting the column's relative yaw from the absolute angle to the matched BEV pixel (Fig.~\ref{fig:angle_est}-b). A single match provides only a point-wise vote; radial lines emerge from consensus. Distinctive structures (roads, trees, building edges) span multiple co-radial BEV pixels sharing the same matched column, and accumulate votes for the same yaw along the radial line regardless of the camera's position along that direction (Fig.~\ref{fig:angle_est}-c).

\begin{proposition}[Position-Invariant Yaw from Line Correspondence]
\label{prop:radial}
Let a BEV pixel $(u,v)$ match to the ground column $k$ with relative yaw $\mathbf{Y}(k)$. For any two candidate camera pose $(x_1, y_1)$ and $(x_2, y_2)$ that are collinear with $(u,v)$, the estimated yaw $\theta_{\mathrm{est}} = \arctan\!\big(\tfrac{v-y}{u-x}\big) - \mathbf{Y}(k)$ is identical. (Proof in appendix.)
\end{proposition}

This is fundamental to our approach: it \textbf{decouples yaw estimation from camera location}. Without knowing the true 2D position, we vote over all candidate poses; within line, Proposition~\ref{prop:radial} ensures that correct matches vote the same yaw regardless of candidate position (as long as it lies on the corresponding radial line). This enables precise yaw estimation despite location uncertainty. The ground height assumption is removed by the column-wise formulation itself (Sec.~\ref{sec:feature_extraction}): relative yaw depends only on horizontal image position, not on object distance or ground elevation.

Based on these properties, three key components enable precise yaw estimation: \textbf{Column Feature Extraction} (Sec.~\ref{sec:feature_extraction}) extracts column-wise features from the ground image that correspond to unique relative yaw. \textbf{Ground-BEV Matching} (Sec.~\ref{sec:method_matching}) computes feature similarity between ground column and BEV pixel and selects a representative matching ground column for each BEV pixel. \textbf{Pair-wise Yaw Voting} (Sec.~\ref{sec:method_voting}) votes match scores for the estimated yaw for each pair of BEV pixel match and 2D candidate pose to find the most probable yaw. More details are in the appendix.

\subsection{Column Feature Extraction}
\label{sec:feature_extraction}
We extract features from the BEV and ground images for cross-view matching. The ground pixel features are aggregated column-wise, capturing features with unique relative yaw. This relative yaw is later used to estimate yaw from a pair of BEV pixels and 2D pose. Relative yaw is not impacted by ground height or distance, relaxing the ground height assumption.

\paragraph{Pixel Feature Extraction} We employ two U-Net~\citep{unet} with a VGG-16~\citep{VGG-16} encoder to extract $C$-dimensional features for two images following prior works~\citep{Shi_2023Boosting, Wang_2024OVCL}. The BEV feature map is denoted as $\mathbf{F}_\text{bev} \in \mathbb{R}^{C \times H_b \times W_b}$. The ground feature map is extracted at the pixel level $\mathbf{F}_\text{pix} \in \mathbb{R}^{C \times H_g \times W_g}$. We add a 4th channel to the ground image before encoding, the heading embedding~\citet{Wang_2023PureACL}, representing the relative yaw of each pixel.

\paragraph{Column-wise Aggregation} We aggregate features column-wise to obtain a relative yaw-aligned feature. For each column $k$, we weigh and sum the pixel features to extract important features. A fully connected layer $N$ and a softmax compute the pixel feature weights. Given the feature of pixel $(k, j)$ as $\mathbf{F}_\text{pix}(k, j)$: 
\begin{equation}
\mathbf{F}_\text{grd}(k) = 
\sum_{j=1}^{H_g} 
\frac{\exp\big(N(\mathbf{F}_\text{pix}(k, j))\big)}
{\sum_{j'=1}^{H_g} \exp\big(N(\mathbf{F}_\text{pix}(k, j'))\big)} 
\, \mathbf{F}_\text{pix}(k, j)
\end{equation}
This results in a column-wise feature $\mathbf{F}_\text{grd} \in \mathbb{R}^{C \times W_g}$.

\paragraph{Feature Confidence Estimation} We compute confidence scores for both BEV and ground features to focus on distinct and matchable features. MLP attached to U-Net computes confidence from $\mathbf{F}_\text{bev}$ and $\mathbf{F}_\text{pix}$, producing $\mathbf{C}_\text{bev} \in \mathbb{R}^{H_b \times W_b}$ and $\mathbf{C}_\text{pix} \in \mathbb{R}^{H_g \times W_g}$. For each column, the highest pixel confidence is selected as $\mathbf{C}_\text{grd} \in \mathbb{R}^{W_g}$. These confidence scores are used in the matching process to improve robustness by assigning greater weight to more confident features.

\subsection{Ground-BEV Matching}
\label{sec:method_matching}

We compute cosine similarity scores for each pair of a ground column and a BEV pixel. One representative ground column is selected per BEV pixel to ensure an effective match. This process does not use projection. Instead, we retain a relative yaw from the ground column, assigning a match score to each BEV pixel. This is used to vote yaw in Sec.~\ref{sec:method_voting} for each BEV-pixel match and 2D pose.

\paragraph{Match Score Computation} Match score is absolute cosine similarity between the BEV feature at pixel $(u, v)$ and each ground column feature. Given BEV feature $\mathbf{F}_\text{bev}(u, v)$ and $k$-th ground column feature $\mathbf{F}_\text{grd}(k)$, the pairwise match score is:
\begin{equation}
\mathbf{P}((u, v), k) = \left| \mathbf{F}_\text{bev}(u, v) \cdot \mathbf{F}_\text{grd}(k) \right| \cdot \mathbf{C}_\text{grd}(k)
\end{equation}
where $|\cdot|$ denotes the absolute value of the dot product, which yields a scalar since features are $C$-dimensional unit vectors after channel-wise normalization. Outer $\cdot$ denotes the scalar multiplication.

\paragraph{Probabilistic Feature Selection} Column selection differs between training and inference. During training, we use probabilistic selection for the model to learn from diverse column-pixel correspondences, preventing dominant features from overshadowing valid ground-truth matches. In inference, we switch to deterministic (hard) selection to pick the single best match for efficiency.

For each BEV pixel $(u,v)$, we compute a selection distribution over all columns based on their match scores, then sample according to the phase:
\begin{equation}
\text{Selection weights: } w_k = \frac{\mathbf{P}((u,v),k)}{\sum_{k'=1}^{W_g} \mathbf{P}((u,v),k')}
\end{equation}
where the match scores $\mathbf{P}((u,v),k)$ are normalized into a probability distribution across the $W_g$ columns. The chosen column index $\textbf{c}(u,v)$ is then:
\begin{equation}
\textbf{c}(u,v) =
\begin{cases}
\text{sample from } \{1, \ldots, W_g\} \text{ with probabilities } w_k, & \text{if Training}, \\
\argmax\limits_k \mathbf{P}((u, v), k), & \text{if Inference}
\end{cases}
\end{equation}
Training-time randomization mitigates overfitting to dominant features and prevents failure cases in which no feature votes for the ground-truth yaw. Without the absolute value, the model collapses to predicting a negative match score for anti-correlated features in non-matching pairs, which is an easier target. Taking the absolute value closes this shortcut while preserving matching quality, at the cost of discarding feature sign. The final BEV pixel match score is computed as:
\begin{equation}
\mathbf{M_s}(u, v) = \mathbf{C}_\text{bev}(u, v) \cdot \mathbf{P}((u, v), \textbf{c}(u,v))
\end{equation}
where $\mathbf{C}_\text{bev}(u, v)$ and $\mathbf{P}((u, v), \textbf{c}(u,v))$ are scalar confidence and match scores for the pixel (positive by construction), $\cdot$ is a scalar multiplication. No absolute value is required here, since both operands are already nonnegative.

\paragraph{Relative Yaw Mapping} Each ground column $k$ corresponds to a unique relative yaw $\mathbf{Y}(k)$, computed from the 3D ray direction in the camera's coordinate frame. The selected column's relative yaw is assign it to the relative yaw map $\mathbf{M_r}$:
\begin{equation}
\mathbf{M_r}(u, v) = \mathbf{Y}(\textbf{c}(u,v)).
\end{equation}
Each BEV pixel $(u, v)$ has an associated match score $\mathbf{M_s}(u, v)$ and a relative yaw $\mathbf{M_r}(u, v)$. These values are utilized in the Yaw Voting to estimate yaw.

\subsection{Pair-wise Yaw Voting}
\label{sec:method_voting}
Yaw estimation is ambiguous without a known 2D pose $(x,y)$. To resolve this ambiguity, we construct a 3D score bin $S(\theta,(x,y))$ over the full pose search space. We iterate over all candidate 2D poses. For each candidate $(x, y)$, we compute the conditional yaw that aligns each matched BEV pixel $(u,v)$ with the radial line of the ground column, and cast its match score $M_s(u, v)$ into the corresponding yaw bin $S(\theta_{\text{est}}, (x, y))$. This process isolates yaw through a relative yaw-based formulation (Fig.~\ref{fig:angle_est}-(b)), where correct BEV pixel matches sharing the same ground column agree on the same yaw for all $(x,y)$ candidates along that radial direction (Fig.~\ref{fig:angle_est}-(c)). This radial consensus removes the need for height-assuming projection and enables sub-degree yaw estimation under location ambiguity.

\paragraph{Yaw Estimation} The absolute angle $\theta_{\text{abs}}$ in BEV coordinate for each candidate pose $(x, y)$ relative to a matched BEV pixel $(u, v)$ is computed by:
\begin{equation}
\theta_{\text{abs}}((x, y), (u, v)) = \arctan\left(\frac{v - y}{u - x}\right)
\label{eq:yawest}
\end{equation}
The estimated yaw is then computed by subtracting the matched ground-view relative yaw $\mathbf{M_r}(u, v)$ from the absolute angle, as visualized in Fig.~\ref{fig:angle_est}:
\begin{equation}
\theta_{\text{est}}((x, y), (u, v)) = \theta_{\text{abs}}((x, y), (u, v)) - \mathbf{M_r}(u, v)
\end{equation}
Crucially, when the 2D position $(x,y)$ lies on a radial line from pixel $(u,v)$, the absolute angle $\theta_{\text{abs}}$ remains constant: all positions on the same radial line from the camera toward $(u,v)$ share the same angular direction, so $\arctan\left(\frac{v - y}{u - x}\right)$ is invariant. Therefore, $\theta_{\text{est}} = \theta_{\text{abs}} - Y(k)$ is also constant along the entire line. This radial invariance is the key property that enables robust yaw estimation under location uncertainty (Proposition~\ref{prop:radial}): correct matches vote for the same yaw regardless of where the camera actually sits on that radial direction.

\paragraph{Yaw Voting}
To handle location uncertainty, we construct a 3D score tensor that votes over all candidate 2D positions and all yaw angles: $\mathbf{S} \in \mathbb{R}^{|\Theta| \times |\mathcal{X}| \times |\mathcal{Y}|}$, where each bin $(\theta, (x, y))$ accumulates votes. The radial invariance from Proposition~\ref{prop:radial} directly enables this 3D approach. When a BEV-to-column match is evaluated at candidate position $(x_1, y_1)$, it votes for yaw $\theta_1$; the same match evaluated at nearby position $(x_2, y_2)$ on the same radial line votes for the identical yaw $\theta_1$. Thus, correct matches accumulate votes across multiple position bins along the radial line. From multiple BEV matching consensus along a radial line, the true yaw forms a strong peak in the 3D tensor, while incorrect yaws receive scattered votes across disconnected position-angle combinations.

For each candidate pose $(x,y)$, we compute the yaw that each BEV pixel vote contributes to (Equation~\ref{eq:yawest}) and accumulate match score into the corresponding bin. The match score $\mathbf{M_s}(u, v)$ is accumulated to the corresponding yaw bin of $\theta$ for each candidate pose $(x,y)$, if $(u,v)$ is within distance $R$ from $(x,y)$:
\begin{equation}
\mathbf{S_M}(\theta, (x,y)) =
\sum_{u, v} \mathbf{M_s}(u, v) \cdot 
\mathbb{I}\!\Big( \theta = [\theta_{\mathrm{est}}((x, y), (u, v))] \,\wedge\, d((x,y),(u,v)) \le R \Big),
\end{equation}
where $\mathbb{I}$ is an indicator function ensuring voting only for matching yaw bins, $d((x,y),(u,v)) = \sqrt{(x-u)^2 + (y-v)^2}$. We set $R$ to the largest radius such that a circle of radius $R$ centered at any candidate pose fits entirely within the BEV image; this prevents boundary poses from receiving fewer votes. (details in the appendix) Although the discrete bin assignment is non-differentiable, the match scores $\mathbf{M_s}$ accumulated in each bin are fully differentiable, allowing gradients to flow back to the feature extractors and confidence estimators.

\paragraph{Pose Plausibility} Different poses have different likelihoods due to the presence of roads and structures. We add a pose plausibility score $\mathbf{S_P}(\theta, (x, y))$ derived from BEV features using an MLP that provides a hint for estimation, inspired by prior work~\citet{Shi_2023Boosting}
\begin{equation}
\mathbf{S}(\theta, (x, y)) = \mathbf{S_P}(\theta, (x, y)) + \mathbf{S_M}(\theta, (x, y))
\end{equation}

\paragraph{Multi Level Score} We use a 3-level estimate, with different resolutions~\citet{Shi_2023Boosting}. One lower level has half the resolution, with a halved number of angular bins. Each level uses a per-pixel 2D pose, and the highest level uses a 1-degree angle bin. Lower-resolution scores are interpolated and added to higher-resolution scores. The highest-resolution is used for the final estimate.

\paragraph{Loss Function} Applying softmax over the final score tensor $\mathbf{S}$ for each resolution, we maximize the log probability of the ground-truth pose bin. For ground-truth pose $(x_\text{gt}, y_\text{gt}, \theta_\text{gt})$:
\begin{equation}
\mathcal{L} = -\log\frac{\exp\big(\mathbf{S}(x_\text{gt}, y_\text{gt}, \theta_\text{gt})\big)}
{\sum_{x \in \mathcal{X}} \sum_{y \in \mathcal{Y}} \sum_{\theta \in \Theta} \exp\big(\mathbf{S}(x, y, \theta)\big)}
\end{equation}
\begin{figure}[!ht]
\setlength{\abovecaptionskip}{0pt}
    \setlength{\belowcaptionskip}{0pt}
    \includegraphics[width=\linewidth]{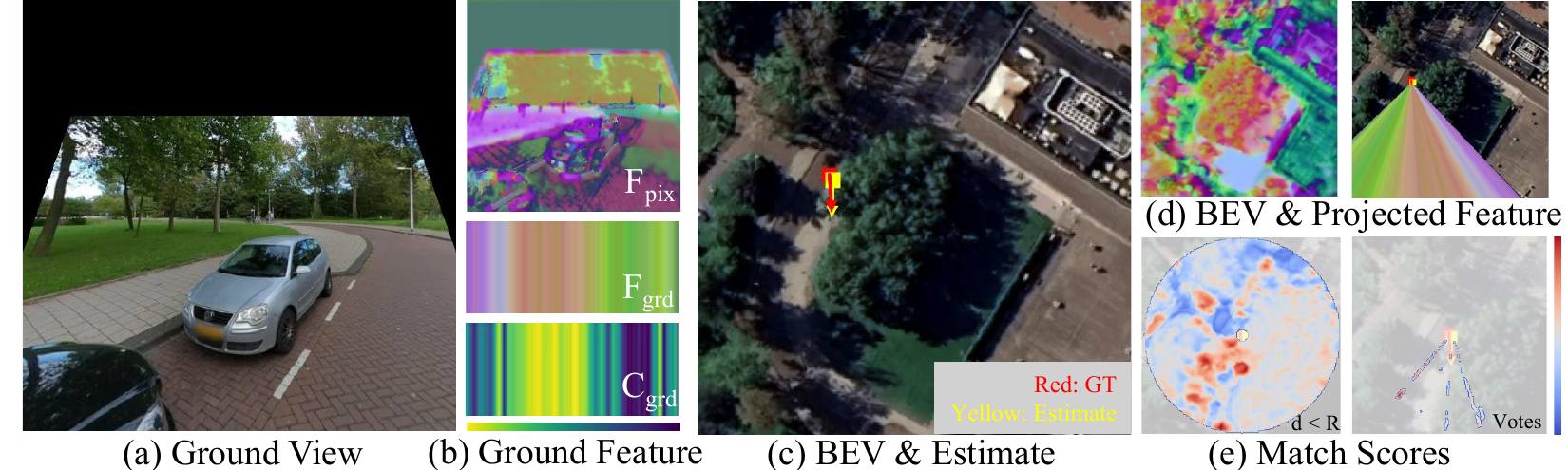}
    \caption{Visualization of \ourmethod~in the MGL dataset. PCA visualizes feature space. \textbf{(a)} is the ground view. \textbf{(b)} are the ground pixel, column features, and confidence: $\textbf{F}_\text{pix}$, $\textbf{F}_\text{grd}$, and $\textbf{C}_\text{grd}$, with a colormap for confidence. \textbf{(c)} is the BEV image with the estimation result. \textbf{(d)} are the BEV and the projected ground feature for the estimate pose (for visualization, not used by our method). \textbf{(e)} are match scores $M_s$ within an $R$ distance of the estimated pose and ones that vote for the estimated yaw, with a color map for the voted score scale. Confident votes are concentrated on a line.
    \textbf{More examples from all datasets are in the appendix.}}
    \label{fig:visualize}
\end{figure}
\section{Experiments}
\subsection{Experiment Setup and Metrics}
To evaluate the method under representative real-world conditions in which GPS and sensor data can be noisy, we perturb the ground-truth poses in both position and orientation, following the protocols of prior work~\citep {xia2023_CCVPE, Shi_2023Boosting}. Yet, under location uncertainty, we focus on assessing the angular accuracy for yaw.

The angular accuracy metric, commonly used for prior works~\cite{Shi_2023Boosting, xia2023_CCVPE, Wang_2023PureACL}, is the ratio of the correct yaw estimate, where the estimation is correct if the yaw angle error is less than $\theta$ degrees relative to the ground truth. Consistent with previous studies~\citep{Wang_2023PureACL, Wang_2024OVCL}, we use thresholds of $\theta = 1^\circ$, $2^\circ$, and $4^\circ$.

\subsection{Baselines}
We compare \ourmethod~against four state-of-the-art baselines with official code: CCVPE~\citep{xia2023_CCVPE}, BoostAcc~\citep{Shi_2023Boosting}, G2S~\citep{g2sweakly_2024_shi}, and FG2~\citep{Xia_2025_CVPR_FG2}. CCVPE utilizes rolling descriptors with orientation-map-based pose estimation. BoostAcc and G2S utilize an optimizer and Spatial Transformer Networks to estimate 3-DoF poses, taking only yaw to compute correlations for 2-DoF poses. G2S uses self-supervised learning with only BEV for yaw estimation. FG2 estimates a 3-DoF pose jointly by weighted matching of the 3D grid-mapped ground view and the BEV, using Procrustes analysis.

All baselines and our method use identical training data, splits, and evaluation protocol. The baselines originally report only a small yaw error (e.g., $\pm10^\circ$), except for VIGOR, where an omnidirectional camera enables estimation of the unknown yaw. Official codes are used to retrain the baselines for unknown yaw setups, with dataloaders adapted accordingly. Some baselines already have the KITTI dataloader; we only changed the rotation error range. For the VIGOR dataset, we use their public weights for evaluation, unless noted as retrained due to unavailability. Despite using official implementation, most methods struggled with a common limited-FoV camera with unknown yaw. We clarify details in the appendix on how each method-dataset pair is adapted.

\subsection{Dataset}
\textbf{MGL~(Mapillary Geo-Localization) Dataset}~\citep{sarlin2023orienternet}\footnote{All images are publicly available under a CC-BY-SA license via the Mapillary API~\citep{sarlin2023orienternet}.} is a crowd-sourced dataset with ground images from vehicles, handheld cameras, and bicycles. Many views are not front-aligned with the road direction, introducing a common challenge for AR/MR applications where orientation alignment is crucial. SfM-refined poses are provided with gravity-aligned images. We augment the Amsterdam split with Google Maps, which provides ground view at a constant resolution consistent with cross-view localization datasets (27,159 train / 9,103 test, 512$\times$512). Further details in the appendix.

\textbf{KITTI Dataset}~\citep{Geiger2013KITTI} is a widely used cross-view localization benchmark with one training set and two test sets (same-area and cross-area). Following prior work~\citet{Shi_2023Boosting}, images are resized to match the Ford dataset. Up to $\pm20$m location noise and $\pm180^{\circ}$ yaw noise are injected.

\textbf{Ford Multi-AV Dataset}~\citep{Ford-Multi-AV} evaluates driving scenes with satellite imagery from Google Map~\citep{Shi_2023Boosting}. We use log~1 (highway), where small angular errors are critical. We follow prior works~\citet{Shi_2023Boosting}, using their splits, ground images are resized to 256$\times$1024, and satellite images are cropped to 512$\times$512. Consistent with prior work~\cite{Wang_2024OVCL}, up to $\pm20$m location noise and $\pm45^{\circ}$ yaw noise is injected.

\textbf{VIGOR Dataset}~\citep{zhu2021vigor} contains ground panorama images of four cities. We follow standard fine-grained protocols~\citep{Lentsch_2023SliceMatch, xia2023_CCVPE} using positive samples and locational noise (center half width and height of BEV). Two splits are used: same-area (all cities for train/test) and cross-area (two cities each). We use the unknown orientation setup with $\pm180^{\circ}$ yaw uncertainty.

\subsection{Results}
\label{sec:main_results}

Table~\ref{tab:comparison_mgl} is the evaluation result for the \textbf{MGL} dataset, where urban variability introduces height ambiguity that projection-based methods cannot handle. \ourmethod~achieves 72.10\% sub-degree accuracy under $\pm45^{\circ}$ noise, doubling G2S (32.67\%). Unknown yaw is particularly challenging, as methods are typically evaluated with slight angular noise unless given a $360^{\circ}$ panoramic image (e.g., VIGOR). Existing methods largely fail, while our method achieves 34.81\% sub-degree accuracy versus 6.55\% for CCVPE.

Table~\ref{tab:comparison_kitti} shows results on the \textbf{KITTI} dataset with $\pm180^{\circ}$ yaw noise. Most methods perform well under small noise ($\pm10^{\circ}$) but break down when the yaw is unknown. Our method achieves 51.02\% and 48.46\% sub-degree accuracy on same- and cross-area splits, while CCVPE achieves 8.96\% and 3.14\%. Performance remains consistent across areas, indicating robustness to unseen regions. Finer-grained $1$--$5^{\circ}$ threshold accuracies for MGL and KITTI are reported in the supplementary material.

Table~\ref{tab:comparison_ford} is the \textbf{Ford Multi-AV} highway dataset result. \ourmethod~achieves 67.05\% sub-degree accuracy, improving 25.69\% over the next best method. Following prior work, it is trained and tested without accounting for pitch and roll, yet stays robust to sensory noise.

\begin{table}[!ht]
\centering
\caption{Comparison of cross-view yaw estimation accuracy on MGL Dataset with $\pm$20m location noise under $\pm45^{\circ}$ and $\pm180^{\circ}$ yaw noise.
}
\label{tab:comparison_mgl}
\begin{tabular}{|l|ccc|ccc|}
\hline
\multirow{2}{*}{Method} & \multicolumn{3}{c|}{$\pm45^{\circ}$} & \multicolumn{3}{c|}{$\pm180^{\circ}$} \\
 & $<1^{\circ}$ & $<2^{\circ}$ & $<4^{\circ}$ & $<1^{\circ}$ & $<2^{\circ}$ & $<4^{\circ}$ \\
\hline
CCVPE~\citep{xia2023_CCVPE} & 21.36 & 40.78 & 68.15 & 6.55 & 12.55 & 23.79 \\
BoostAcc~\citep{Shi_2023Boosting} & 7.40 & 14.19 & 27.54 & 0.58 & 1.09 & 2.13 \\
G2S~\citep{g2sweakly_2024_shi} & 32.67 & 57.49 & 82.50 & 0.59 & 1.27 & 1.86 \\
FG2~\citep{Xia_2025_CVPR_FG2} & 18.59 & 34.97 & 59.20 & 3.13 & 6.28 & 12.44 \\
\textbf{Ours} & \textbf{72.10} & \textbf{90.16} & \textbf{95.63} & \textbf{34.81} & \textbf{52.42} & \textbf{61.74} \\
\hline
\end{tabular}
\end{table}

\begin{table}[!ht]
\centering
\caption{Comparison of cross-view yaw estimation accuracy on KITTI dataset with $\pm$20m location noise and unknown yaw. *: CCVPE reports accuracy at $<1^{\circ}$, $<3^{\circ}$, and $<5^{\circ}$ thresholds~\citep{xia2023_CCVPE}; $\le$ entries denote that the published value is from a looser threshold ($<3^{\circ}$ for $<2^{\circ}$, $<5^{\circ}$ for $<4^{\circ}$), so true accuracy is equal or lower.}
\label{tab:comparison_kitti}
\begin{tabular}{|l|ccc|ccc|}
\hline
\multirow{2}{*}{Method} & \multicolumn{3}{c|}{Test 1 (Same Area)} & \multicolumn{3}{c|}{Test 2 (Cross Area)} \\
 & $<1^{\circ}$ & $<2^{\circ}$ & $<4^{\circ}$ & $<1^{\circ}$ & $<2^{\circ}$ & $<4^{\circ}$ \\
\hline
CCVPE*~\citep{xia2023_CCVPE} & 8.96 & $\le$26.48 & $\le$42.75 & 3.14 & $\le$9.24 & $\le$14.56 \\
BoostAcc~\citep{Shi_2023Boosting} & 0.53 & 0.90 & 1.96 & 0.56 & 1.05 & 2.09 \\
G2S~\citep{g2sweakly_2024_shi} & 0.58 & 1.17 & 2.17 & 0.54 & 0.82 & 1.64 \\
FG2~\citep{Xia_2025_CVPR_FG2} & 1.62 & 3.15 & 6.47 & 1.62 & 2.96 & 6.13 \\
\textbf{Ours} & \textbf{52.98} & \textbf{65.07} & \textbf{67.72} & \textbf{49.59} & \textbf{58.92} & \textbf{60.25} \\
\hline
\end{tabular}%
\end{table}

\begin{table}[!ht]
\centering
\caption{Comparison of cross-view yaw estimation accuracy on VIGOR Dataset (omnidirectional camera), $\pm0.25$ BEV width and height location noise, and unknown yaw.
*:~BoostAcc did not evaluate on VIGOR.
**:~G2S evaluated only the translation estimator on VIGOR.
$\dagger$: FG2 provides VIGOR result with a two-stage approach.}
\label{tab:comparison_vigor}
\begin{tabular}{|l|ccc|ccc|}
\hline
\multirow{2}{*}{Method} & \multicolumn{3}{c|}{Same Area} & \multicolumn{3}{c|}{Cross Area} \\
 & $<1^{\circ}$ & $<2^{\circ}$ & $<4^{\circ}$ & $<1^{\circ}$ & $<2^{\circ}$ & $<4^{\circ}$ \\
\hline
CCVPE~\citep{xia2023_CCVPE} & 8.31 & 16.69 & 32.40 & 8.96 & 17.66 & 34.34 \\
BoostAcc~\citep{Shi_2023Boosting}* & 0.50 & 1.10 & 2.16 & 0.64 & 1.20 & 2.30 \\
G2S~\citep{g2sweakly_2024_shi}** & 2.33 & 4.60 & 8.92 & 3.98 & 7.72 & 14.45 \\
FG2~\citep{Xia_2025_CVPR_FG2} & 19.05 & 35.94 & 58.91 & 9.88 & 19.37 & 35.57 \\
FG2~\citep{Xia_2025_CVPR_FG2}$\dagger$ & 20.78 & 38.17 & \textbf{62.11} & 12.39 & 23.48 & 41.40 \\
\textbf{Ours} & \textbf{34.16} & \textbf{50.27} & 59.32 & \textbf{23.67} & \textbf{37.09} & \textbf{47.15} \\
\hline
\end{tabular}
\end{table}

Table~\ref{tab:comparison_vigor} shows results on the \textbf{VIGOR} dataset with omnidirectional ground images. The panoramic view provides more information, so prior methods remain competitive.
At the $<4^{\circ}$ threshold in the same area, FG2's two-stage variant slightly outperforms ours, as our voting concentrates scores in narrow angular peaks optimized for sub-degree precision. At the stricter $<1^{\circ}$ and $<2^{\circ}$ thresholds, \ourmethod~leads by 13.38\% and 12.10\%, and at all cross-area thresholds.

\subsection{Analysis}
\subsubsection{Matching Process and Location}
\label{sec:analysis}
Fig.~\ref{fig:visualize} and appendix examples across all datasets illustrate the matching process. High-scoring matches arise along short co-linear pixels and vote the same yaw for all locations along the radial line~(as in Fig.~\ref{fig:angle_est}-c). The voting is agnostic to the physical nature of matched structures; appendix visualizations confirm correspondences among road segments, tree branches, and building facades across diverse scenes. Overall, our formulation distributes scores along a radial line rather than concentrating them at a single pixel. This is by design: the method learns the most confident line correspondences for yaw. Our method can instead support 3-DoF localization methods as shown in Sec.~\ref{sec:angular_prior}, reducing their search space to 2-DoF.

Fig.~\ref{fig:match_evidence_vigor} traces a single VIGOR cross-area prediction through the matching pipeline. Different columns focus on different cues: two lock onto road segments and produce the largest match scores, one onto a tree, and one onto a grass patch in front of a building. The matching-only $S_M$ curve peaks at the correct yaw, with a smaller peak near the opposite ($180^{\circ}$) direction.

\begin{figure}[t]
    \centering
    \includegraphics[width=0.95\linewidth]{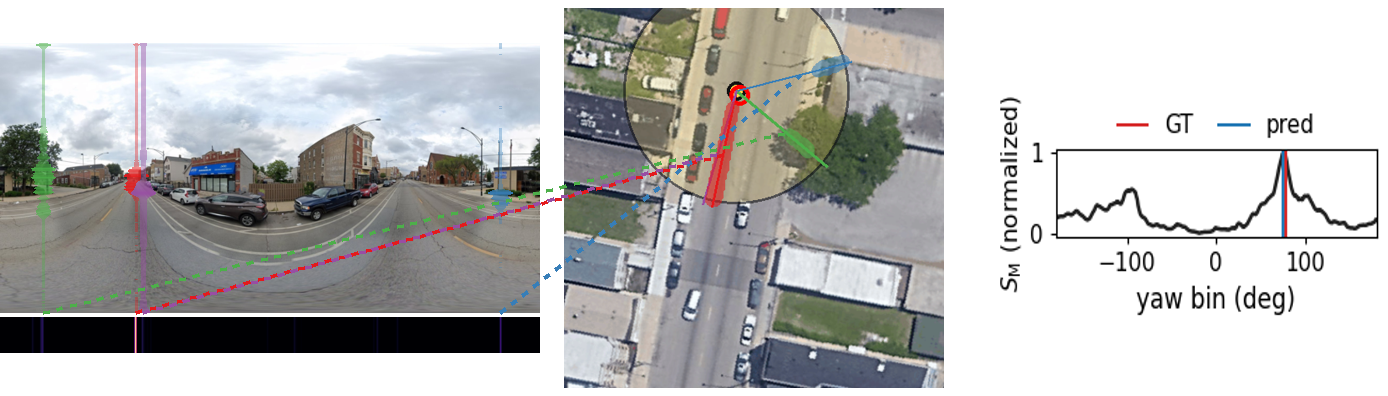}
    \caption{Matching visualization for one VIGOR cross-area sample. \emph{Ground view (top):} the four columns that contribute most to the predicted yaw, with bar length proportional to the per-pixel saliency $|\nabla\cdot I|$ of the matching score w.r.t.\ the input pixels.
    \emph{BEV (left):} thin colored lines mark each column's view direction from the predicted pose; circles on every pixel show the score each column votes into the predicted yaw bin, with circle area proportional to match score. Yellow and red markers denote the predicted and the GT pose.
    \emph{Match score curve (right):} $S_M$ over the full $360^{\circ}$ yaw range is plotted.}
    \label{fig:match_evidence_vigor}
\end{figure}

\subsubsection{Ablation Study}
\label{sec:ablation}
Table~\ref{tab:ablation} ablates key components. In \textbf{Point-to-Point Match}, we replace radially invariant scoring with projection-based 3-DoF scoring that maps ground pixels to BEV for each pose candidate; yielding decent results yet significantly behind radially invariant scoring. In \textbf{Polar-Transform}, we adapt OrienterNet\citet{sarlin2023orienternet}'s polar-transform-based projective 3-DoF pose scoring to satellite view input, and their flexible mapping overfits compared to the naive 3-DoF scoring, demonstrating our formulation's advantage over existing 3-DoF scoring approaches. We analyze Column Feature Extraction (Sec.~\ref{sec:feature_extraction}). In \textbf{Keypoint-based}, we use a keypoint feature extractor~\citep{Wang_2023PureACL}, which leads to poor performance because keypoints often miss important features. \textbf{Column MLP} uses MLP to aggregate pixel features per column. The weighted-sum formulation yields better accuracy. Finally, we show the effectiveness of the hybrid BEV-ground matching. In \textbf{Greedy Selection}, we match each BEV pixel to the ground feature with the highest correlation score during training, where dominant features overshadow others and limit information diversity. \textbf{Test-time Weighted} uses probabilistic matching during the test-time. While yielding decent results, the highest-scoring match proved more effective during inference.

\subsubsection{Accurate Yaw's Impact on Localization Method}
\label{sec:angular_prior}

We show that isolating yaw from the full 3-DoF problem is effective even when the localization method does not break down. Tables~\ref{tab:ours_as_initializer} and~\ref{tab:lays_as_prior_extended} report the localization error of FG2~\citep{Xia_2025_CVPR_FG2} when its first-stage yaw is replaced by different estimators. \textbf{+ Two-stage} initializes FG2's second-stage localizer with FG2's own first-stage yaw; \textbf{+ Ours init.} substitutes \ourmethod{}'s yaw, and the VIGOR \textbf{(2-DoF)} variant additionally fixes our yaw and solves only 2D location via Procrustes analysis. On VIGOR (Table~\ref{tab:ours_as_initializer}), where FG2's yaw is already functional, our 2-DoF variant still improves localization by $14\%$ over FG2's two-stage ($5.95\rightarrow5.10$), approaching the ground-truth-yaw oracle ($2.41$). The benefit grows as the base yaw becomes unreliable (Table~\ref{tab:lays_as_prior_extended}): even under mild noise on MGL the substitution more than halves the median error ($8.95\,\mathrm{m}\rightarrow3.99\,\mathrm{m}$), and under unknown yaw on KITTI, where the two-stage FG2 stays near random, our yaw makes precise localization possible ($12.26\,\mathrm{m}\rightarrow1.48\,\mathrm{m}$ median). Isolating yaw is therefore most critical in the hardest cases.

\begin{table}[t]
\centering
\begin{minipage}{0.42\linewidth}
\centering
\caption{FG2 localization error (m) on cross-area VIGOR with different first-stage yaw initializers.}
\label{tab:ours_as_initializer}
\setlength{\tabcolsep}{4pt}
\begin{tabular}{|l|c|}
\hline
First-stage yaw & Mean$\downarrow$ \\
\hline
FG2                           & 10.02 \\
+ Two-stage                   & 5.95 \\
+ \textbf{Ours init.}         & 5.43 \\
+ \textbf{Ours init.} (2-DoF) & \textbf{5.10} \\
\hline
+ GT Yaw (oracle)             & 2.41 \\
\hline
\end{tabular}
\end{minipage}\hfill
\begin{minipage}{0.55\linewidth}
\centering
\caption{FG2 localization error (m) on MGL and KITTI with different first-stage yaw initializers.}
\label{tab:lays_as_prior_extended}
\setlength{\tabcolsep}{4pt}
\begin{tabular}{|l|cc|}
\hline
First-stage yaw & Mean$\downarrow$ & Median$\downarrow$ \\
\hline
\multicolumn{3}{|l|}{\textit{MGL} ($\pm45^{\circ}$)} \\
\hline
FG2                   & 10.53 & 8.90 \\
+ Two-stage           & 10.54 & 8.95 \\
+ \textbf{Ours init.} & \textbf{6.42} & \textbf{3.99} \\
\hline
\multicolumn{3}{|l|}{\textit{KITTI} ($\pm180^{\circ}$)} \\
\hline
FG2                   & 14.32 & 14.40 \\
+ Two-stage           & 14.44 & 12.26 \\
+ \textbf{Ours init.} & \textbf{2.52} & \textbf{1.48} \\
\hline
\end{tabular}
\end{minipage}
\end{table}

\subsubsection{Robustness Against Pitch and Roll Noise}
\label{sec:robustness_pr_main}
Our column-wise formulation is inherently insensitive to pitch errors, as features are aggregated per column rather than per pixel row. Roll is handled at the feature-encoding level, since the roll-tilted horizon is clear in the input image. Table~\ref{tab:pitchroll_impact_main} confirms our method stays robust under $\pm10^\circ$ pitch and roll noise on MGL, with minimal degradation.

\begin{table}[!ht]
\centering
\begin{minipage}{0.48\linewidth}
\centering
\caption{Comparison of cross-view yaw estimation accuracy on Ford Dataset Log1 (Highway) with $\pm$20m location noise and $\pm45^{\circ}$ yaw noise.}
\label{tab:comparison_ford}
\begin{tabular}{|l|ccc|}
\hline
Method & \textless1$^{\circ}$ & \textless2$^{\circ}$ & \textless4$^{\circ}$ \\
\hline
CCVPE~\citep{xia2023_CCVPE} & 41.36 & 62.74 & 75.11 \\
BoostAcc~\citep{Shi_2023Boosting} & 24.38 & 43.24 & 67.86 \\
G2S~\citep{g2sweakly_2024_shi} & 16.81 & 47.71 & 68.43 \\
FG2~\citep{Xia_2025_CVPR_FG2} & 16.75 & 41.84 & 57.13 \\
\textbf{Ours} & \textbf{67.05} & \textbf{85.14} & \textbf{91.76} \\
\hline
\end{tabular}
\end{minipage}\hfill
\begin{minipage}{0.48\linewidth}
\centering
\caption{Impact of components in \ourmethod~for yaw estimation with $180^{\circ}$ yaw ambiguity in MGL dataset.}
\label{tab:ablation}
\begin{tabular}{|l|ccc|}
\hline
Alternative Method & \textless1$^{\circ}$ & \textless2$^{\circ}$ & \textless4$^{\circ}$ \\
\hline
Point-to-Point Match & 13.19 & 23.67 & 38.88 \\
Polar-Transform~\citep{sarlin2023orienternet} & 7.59 & 14.53 & 23.74 \\
\hline
Keypoint-based & 20.12 & 28.75 & 33.10 \\
Column MLP & 31.73 & 48.23 & 58.13 \\
\hline
Greedy Selection & 28.69 & 41.70 & 47.62 \\
Test-time Weighted & 30.77 & 44.72 & 51.33 \\
\hline
\textbf{Ours} & \textbf{34.81} & \textbf{52.42} & \textbf{61.74} \\
\hline
\end{tabular}
\end{minipage}
\end{table}

\begin{table}[!ht]
\centering
\caption{Impact of pitch and roll noise on cross-view yaw estimation accuracy on MGL Dataset with 20m location noise and unknown yaw.}
\label{tab:pitchroll_impact_main}
\begin{tabular}{|l|ccc|}
\hline
Pitch, Roll Noise & \textless1$^{\circ}$ & \textless2$^{\circ}$ & \textless4$^{\circ}$ \\
\hline
$\pm0^\circ$, $\pm0^\circ$ & 34.81 & 52.42 & 61.74 \\
$\pm10^\circ$, $\pm0^\circ$ & 34.36 & 50.26 & 58.61 \\
$\pm10^\circ$, $\pm10^\circ$ & 33.15 & 48.98 & 58.60 \\
\hline
\end{tabular}
\end{table}

\subsubsection{Computational Cost}
\label{sec:computational_cost}
\ourmethod~runs at 0.16\,s per query for VIGOR~\cite{zhu2021vigor} Dataset with most columns on an NVIDIA RTX A6000, competitive with FG2~\citep{Xia_2025_CVPR_FG2}'s 0.26\,s. The yaw-voting stage dominates but stays comparable to feature extraction (appendix has a breakdown), with room for kernel-level optimization.

\subsubsection{Failure Cases and Limitations}
Our method has two main limitations. First, it cannot provide precise standalone localization, as the voting formulation distributes scores along radial lines rather than at a single location; we address this by using \ourmethod~as a yaw prior for downstream 3-DoF methods (Sec.~\ref{sec:angular_prior}). Second, performance degrades in scenes with highly uniform texture (e.g., open fields), rare in practice. The appendix provides a detailed failure analysis.

\section{Conclusion}
We introduced \ourmethod, a sub-degree yaw estimator based on geometric line-alignment consensus between the ground view and the BEV. Our key insight is that a single radial correspondence is sufficient to determine yaw, even when the camera position is unknown: pairwise voting across all candidate positions makes the formulation radially invariant, eliminating the ground-height or known-location assumptions that prior methods require. Experiments on Mapillary, Ford, KITTI, and VIGOR show consistent state-of-the-art gains across diverse scenarios, from urban imagery to highway driving. By isolating yaw from location, \ourmethod~shrinks the localization search space and improves downstream 3-DoF accuracy, opening directions for integrating yaw priors into broader localization pipelines for navigation and AR/MR applications.

\paragraph{Acknowledgments}
This work was supported by the National Research Foundation of Korea(NRF) grant funded by the Korea government(MSIT) (No. RS-2024-00463802, No. RS-2022-NR070595).



\bibliographystyle{splncs04}
\bibliography{main}

@String(CVPR  = {IEEE Conf. Comput. Vis. Pattern Recog.})

@String(ICCV  = {Int. Conf. Comput. Vis.})

@String(ECCV  = {Eur. Conf. Comput. Vis.})

@String(NeurIPS = {Adv. Neural Inform. Process. Syst.})

@String(ACCV  = {Asian Conf. Comput. Vis.})

@String(CVPR  = {CVPR})

@String(ICCV  = {ICCV})

@String(ECCV  = {ECCV})

@String(NeurIPS = {NeurIPS})

@String(ACCV  = {ACCV})

@misc{Ford-Multi-AV,
    title={Ford Multi-AV Seasonal Dataset},
    author={Siddharth Agarwal and Ankit Vora and Gaurav Pandey and Wayne Williams and Helen Kourous and James McBride},
    year={2020},
    eprint={2003.07969},
    archivePrefix={arXiv},
    primaryClass={cs.RO}
}

@inproceedings{Liu2019lending,
  author    = {Liu, Liu and Li, Hongdong},
  title     = {Lending Orientation to Neural Networks for Cross-View Geo-Localization},
  booktitle = {Proceedings of the IEEE/CVF Conference on Computer Vision and Pattern Recognition (CVPR)},
  year      = {2019},
  pages     = {5617--5626}
}

@incollection{NIPS2019_shi, title = {Spatial-Aware Feature Aggregation for Image based Cross-View Geo-Localization}, author = {Shi, Yujiao and Liu, Liu and Yu, Xin and Li, Hongdong}, booktitle = {Advances in Neural Information Processing Systems 32}, editor = {H. Wallach and H. Larochelle and A. Beygelzimer and F. d\textquotesingle Alch'{e}-Buc and E. Fox and R. Garnett}, pages = {10090--10100}, year = {2019}, publisher = {Curran Associates, Inc.}, url = {http://papers.nips.cc/paper/9199-spatial-aware-feature-aggregation-for-image-based-cross-view-geo-localization.pdf} }

@inproceedings{shi2019optimal, title={Optimal Feature Transport for Cross-View Image Geo-Localization}, author={Shi, Yujiao and Yu, Xin and Liu, Liu and Zhang, Tong and Li, Hongdong}, booktitle={arXiv preprint arXiv:1907.05021}, year={2019} }

@INPROCEEDINGS{2020shi,
  author={Shi, Yujiao and Yu, Xin and Campbell, Dylan and Li, Hongdong},
  booktitle={2020 IEEE/CVF Conference on Computer Vision and Pattern Recognition (CVPR)}, 
  title={Where Am I Looking At? Joint Location and Orientation Estimation by Cross-View Matching}, 
  year={2020},
  volume={},
  number={},
  pages={4063-4071},
  doi={10.1109/CVPR42600.2020.00412}}

@inproceedings{zhu2021vigor,
  title={VIGOR: Cross-View Image Geo-localization beyond One-to-one Retrieval},
  author={Zhu, Sijie and Yang, Taojiannan and Chen, Chen},
  booktitle={Proceedings of the IEEE/CVF Conference on Computer Vision and Pattern Recognition},
  pages={3640--3649},
  year={2021}
}

@inproceedings{shi_2022CVPR_beyond,
    title={Beyond Cross-view Image Retrieval: Highly Accurate Vehicle Localization Using Satellite Image}, author={Shi, Yujiao and Li, Hongdong},
    booktitle={Proceedings of the IEEE Conference on Computer Vision and Pattern Recognition},
    year={2022}
}

@InProceedings{Lentsch_2023SliceMatch,
    author    = {Lentsch, Ted and Xia, Zimin and Caesar, Holger and Kooij, Julian F. P.},
    title     = {SliceMatch: Geometry-Guided Aggregation for Cross-View Pose Estimation},
    booktitle = {Proceedings of the IEEE/CVF Conference on Computer Vision and Pattern Recognition (CVPR)},
    month     = {June},
    year      = {2023},
    pages     = {17225-17234}
}

@INPROCEEDINGS{Shi_2023Boosting,
  author={Shi, Yujiao and Wu, Fei and Perincherry, Akhil and Vora, Ankit and Li, Hongdong},
  booktitle={2023 IEEE/CVF International Conference on Computer Vision (ICCV)}, 
  title={Boosting 3-DoF Ground-to-Satellite Camera Localization Accuracy via Geometry-Guided Cross-View Transformer}, 
  year={2023},
  volume={},
  number={},
  pages={21459-21469},
  doi={10.1109/ICCV51070.2023.01967}
}

@inproceedings{Song_2023_LearningFlow,
 author = {Song, Zhenbo and xianghui, ze and Lu, Jianfeng and Shi, Yujiao},
 booktitle = {Advances in Neural Information Processing Systems},
 editor = {A. Oh and T. Naumann and A. Globerson and K. Saenko and M. Hardt and S. Levine},
 pages = {70612--70625},
 publisher = {Curran Associates, Inc.},
 title = {Learning Dense Flow Field for Highly-accurate Cross-view Camera Localization},
 url = {https://proceedings.neurips.cc/paper_files/paper/2023/file/df5f94d6ac6e13d830d70536cde9f0d2-Paper-Conference.pdf},
 volume = {36},
 year = {2023}
}

@inproceedings{NEURIPS2023_homography,
 author = {Wang, Xiaolong and Xu, Runsen and Cui, Zhuofan and Wan, Zeyu and Zhang, Yu},
 booktitle = {Advances in Neural Information Processing Systems},
 editor = {A. Oh and T. Naumann and A. Globerson and K. Saenko and M. Hardt and S. Levine},
 pages = {5301--5319},
 publisher = {Curran Associates, Inc.},
 title = {Fine-Grained Cross-View Geo-Localization Using a Correlation-Aware Homography Estimator},
 url = {https://proceedings.neurips.cc/paper_files/paper/2023/file/112d8e0c7563de6e3408b49a09b4d8a3-Paper-Conference.pdf},
 volume = {36},
 year = {2023}
}

@article{xia2023_CCVPE,
  author={Xia, Zimin and Booij, Olaf and Kooij, Julian F. P.},
  journal={IEEE Transactions on Pattern Analysis and Machine Intelligence}, 
  title={Convolutional Cross-View Pose Estimation}, 
  year={2024},
  volume={46},
  number={5},
  pages={3813-3831},
  doi={10.1109/TPAMI.2023.3346924}
}

@inproceedings{Wang_2023PureACL,
  title={View Consistent Purification for Accurate Cross-View Localization},
  author={Wang, Shan and Zhang, Yanhao and Perincherry, Akhil and Vora, Ankit and Li, Hongdong},
  booktitle={Proceedings of the IEEE/CVF International Conference on Computer Vision},
  pages={8197--8206},
  year={2023}
}

@InProceedings{Wang_2024OVCL,
    author    = {Wang, Shan and Nguyen, Chuong and Liu, Jiawei and Zhang, Yanhao and Muthu, Sundaram and Maken, Fahira Afzal and Zhang, Kaihao and Li, Hongdong},
    title     = {View From Above: Orthogonal-View aware Cross-view Localization},
    booktitle = {Proceedings of the IEEE/CVF Conference on Computer Vision and Pattern Recognition (CVPR)},
    month     = {June},
    year      = {2024},
    pages     = {14843-14852}
}

@InProceedings{Shi_2022_ACCV_CVLNet,
    author    = {Shi, Yujiao and Yu, Xin and Wang, Shan and Li, Hongdong},
    title     = {CVLNet: Cross-View Feature Correspondence Learning for Video-based Camera Localization},
    booktitle = {Proceedings of the Asian Conference on Computer Vision (ACCV)},
    month     = {December},
    year      = {2022},
    pages     = {652-669}
}

@InProceedings{Fervers_2023_CVPR_CSLA,
    author    = {Fervers, Florian and Bullinger, Sebastian and Bodensteiner, Christoph and Arens, Michael and Stiefelhagen, Rainer},
    title     = {Uncertainty-Aware Vision-Based Metric Cross-View Geolocalization},
    booktitle = {Proceedings of the IEEE/CVF Conference on Computer Vision and Pattern Recognition (CVPR)},
    month     = {June},
    year      = {2023},
    pages     = {21621-21631}
}

@INPROCEEDINGS{VGG-16,
  author={Liu, Shuying and Deng, Weihong},
  booktitle={2015 3rd IAPR Asian Conference on Pattern Recognition (ACPR)}, 
  title={Very deep convolutional neural network based image classification using small training sample size}, 
  year={2015},
  volume={},
  number={},
  pages={730-734},
  doi={10.1109/ACPR.2015.7486599}}

@article{Geiger2013KITTI,
  author = {Andreas Geiger and Philip Lenz and Christoph Stiller and Raquel Urtasun},
  title = {Vision meets Robotics: The KITTI Dataset},
  journal = {International Journal of Robotics Research (IJRR)},
  year = {2013}
}

@inproceedings{sarlin2023orienternet,
  author    = {Paul-Edouard Sarlin and
               Daniel DeTone and
               Tsun-Yi Yang and
               Armen Avetisyan and
               Julian Straub and
               Tomasz Malisiewicz and
               Samuel Rota Bulo and
               Richard Newcombe and
               Peter Kontschieder and
               Vasileios Balntas},
  title     = {{OrienterNet: Visual Localization in 2D Public Maps with Neural Matching}},
  booktitle = {CVPR},
  year      = {2023},
}

@inproceedings{2024maplocnet,
  author    = {Wu, Hang and Zhang, Zhenghao and Lin, Siyuan and Mu, Xiangru and Zhao, Qiang and Yang, Ming and Qin, Tong},
  title     = {MapLocNet: Coarse-to-Fine Feature Registration for Visual Re-Localization in Navigation Maps},
  booktitle = {IEEE/RSJ International Conference on Intelligent Robots and Systems (IROS)},
  year      = {2024}
}

@inproceedings{sarlin2023snap,
  author    = {Paul-Edouard Sarlin and
               Eduard Trulls and
               Marc Pollefeys and
               Jan Hosang and
               Simon Lynen},
  title     = {{SNAP: Self-Supervised Neural Maps for Visual Positioning and Semantic Understanding}},
  booktitle = {NeurIPS},
  year      = {2023}
}

@inproceedings{zhu2022transgeo,
  title={TransGeo: Transformer Is All You Need for Cross-view Image Geo-localization},
  author={Zhu, Sijie and Shah, Mubarak and Chen, Chen},
  booktitle={Proceedings of the IEEE/CVF Conference on Computer Vision and Pattern Recognition},
  pages={1162--1171},
  year={2022}
}

@article{2024vpsreview,
author = {Rajpurohit, Aryan and Kumar, Puneet and Singh, Dalwinder and Kumar, Rishu},
year = {2024},
month = {01},
pages = {},
title = {A Review on Visual Positioning System},
journal = {SSRN Electronic Journal},
doi = {10.2139/ssrn.4485458}
}

@Article{mlpcompass2024,
AUTHOR = {Du, Yao and Mateo, Carlos and Tahri, Omar},
TITLE = {A Multilayer Perceptron-Based Spherical Visual Compass Using Global Features},
JOURNAL = {Sensors},
VOLUME = {24},
YEAR = {2024},
NUMBER = {7},
ARTICLE-NUMBER = {2246},
URL = {https://www.mdpi.com/1424-8220/24/7/2246},
PubMedID = {38610457},
ISSN = {1424-8220},
DOI = {10.3390/s24072246}
}

@article{liu2022visual,
  title={A visual compass based on point and line features for UAV high-altitude orientation estimation},
  author={Liu, Ying and Tao, Junyi and Kong, Da and Zhang, Yu and Li, Ping},
  journal={Remote Sensing},
  volume={14},
  number={6},
  pages={1430},
  year={2022},
  publisher={MDPI}
}

@article{Xian2019UprightNetGC,
  title={UprightNet: Geometry-Aware Camera Orientation Estimation From Single Images},
  author={Wenqi Xian and Zhengqi Li and Matthew Fisher and Jonathan Eisenmann and Eli Shechtman and Noah Snavely},
  journal={2019 IEEE/CVF International Conference on Computer Vision (ICCV)},
  year={2019},
  pages={9973-9982},
  url={https://api.semanticscholar.org/CorpusID:201107189}
}

@inproceedings{jin2023perspective,
      title={Perspective Fields for Single Image Camera Calibration},
      author={Linyi Jin and Jianming Zhang and Yannick Hold-Geoffroy and Oliver Wang and Kevin Matzen and Matthew Sticha and David F. Fouhey},
      booktitle = {CVPR},
      year={2023}
}

@inproceedings{2023iccvcamrot,
    author = {Fabien Delattre and David Dirnfeld and Phat Nguyen and Stephen Scarano and Michael J. Jones and Pedro Miraldo and Erik Learned-Miller},
    title = {Robust Frame-to-Frame Camera Rotation Estimation in Crowded Scenes},
    booktitle = {IEEE/CVF International Conference on Computer Vision (ICCV)},
    year = 2023,
    pages = {9752-9762},
    month = 10
}

@Article{2022navvis,
AUTHOR = {Yang, Anbang and Beheshti, Mahya and Hudson, Todd E. and Vedanthan, Rajesh and Riewpaiboon, Wachara and Mongkolwat, Pattanasak and Feng, Chen and Rizzo, John-Ross},
TITLE = {UNav: An Infrastructure-Independent Vision-Based Navigation System for People with Blindness and Low Vision},
JOURNAL = {Sensors},
VOLUME = {22},
YEAR = {2022},
NUMBER = {22},
ARTICLE-NUMBER = {8894},
URL = {https://www.mdpi.com/1424-8220/22/22/8894},
PubMedID = {36433501},
ISSN = {1424-8220},
DOI = {10.3390/s22228894}
}

@article{unet,
  author = {Ronneberger, Olaf and Fischer, Philipp and Brox, Thomas},
  ee = {http://arxiv.org/abs/1505.04597},
  journal = {CoRR},
  title = {U-Net: Convolutional Networks for Biomedical Image Segmentation.},
  url = {http://dblp.uni-trier.de/db/journals/corr/corr1505.html\#RonnebergerFB15},
  volume = {abs/1505.04597},
  year = 2015
}

@ARTICLE{2022softgeo,
  author={Guo, Yulan and Choi, Michael and Li, Kunhong and Boussaid, Farid and Bennamoun, Mohammed},
  journal={IEEE Transactions on Image Processing}, 
  title={Soft Exemplar Highlighting for Cross-View Image-Based Geo-Localization}, 
  year={2022},
  volume={31},
  number={},
  pages={2094-2105},
  doi={10.1109/TIP.2022.3152046}}

@misc{xu2015empiricalevaluationrectifiedactivations,
      title={Empirical Evaluation of Rectified Activations in Convolutional Network}, 
      author={Bing Xu and Naiyan Wang and Tianqi Chen and Mu Li},
      year={2015},
      eprint={1505.00853},
      archivePrefix={arXiv},
      primaryClass={cs.LG},
      url={https://arxiv.org/abs/1505.00853}, 
}

@article{geoloc_survey,
author = {Brejcha, Jan and \v{C}ad\'{\i}k, Martin},
title = {State-of-the-art in visual geo-localization},
year = {2017},
issue_date = {August    2017},
publisher = {Springer-Verlag},
address = {Berlin, Heidelberg},
volume = {20},
number = {3},
issn = {1433-7541},
url = {https://doi.org/10.1007/s10044-017-0611-1},
doi = {10.1007/s10044-017-0611-1},
journal = {Pattern Anal. Appl.},
month = aug,
pages = {613–637},
numpages = {25}
}

@article{Weyand2016PlaNetP,
  title={PlaNet - Photo Geolocation with Convolutional Neural Networks},
  author={Tobias Weyand and Ilya Kostrikov and James Philbin},
  journal={ArXiv},
  year={2016},
  volume={abs/1602.05314},
  url={https://api.semanticscholar.org/CorpusID:171846}
}

@inproceedings{Hays:2008:im2gps,
  Author    = {James Hays and Alexei A. Efros},
  Title     = {im2gps: estimating geographic information from a single image},
  Booktitle = {Proceedings of the {IEEE} Conf. 
               on Computer Vision and Pattern Recognition ({CVPR})},
  Year      = {2008},
}

@INPROCEEDINGS{buildingrome,
  author={Agarwal, Sameer and Snavely, Noah and Simon, Ian and Seitz, Steven M. and Szeliski, Richard},
  booktitle={2009 IEEE 12th International Conference on Computer Vision}, 
  title={Building Rome in a day}, 
  year={2009},
  volume={},
  number={},
  pages={72-79},
  doi={10.1109/ICCV.2009.5459148}}

@InProceedings{worldwidepose,
  author        = {Yunpeng Li and Noah Snavely and Daniel Huttenlocher and
                  Pascal Fua},
  title         = {Worldwide Pose Estimation using 3{D} Point Clouds},
  booktitle     = {European Conf. on Computer Vision},
  year          = {2012}
}

@InProceedings{scalable6-dof,
author="Middelberg, Sven
and Sattler, Torsten
and Untzelmann, Ole
and Kobbelt, Leif",
editor="Fleet, David
and Pajdla, Tomas
and Schiele, Bernt
and Tuytelaars, Tinne",
title="Scalable 6-DOF Localization on Mobile Devices",
booktitle="Computer Vision -- ECCV 2014",
year="2014",
publisher="Springer International Publishing",
address="Cham",
pages="268--283",
isbn="978-3-319-10605-2"
}

@inproceedings{g2sweakly_2024_shi,
author = {Shi, Yujiao and Li, Hongdong and Perincherry, Akhil and Vora, Ankit},
title = {Weakly-Supervised Camera Localization by Ground-to-Satellite Image Registration},
year = {2024},
isbn = {978-3-031-72672-9},
publisher = {Springer-Verlag},
address = {Berlin, Heidelberg},
url = {https://doi.org/10.1007/978-3-031-72673-6_3},
doi = {10.1007/978-3-031-72673-6_3},
booktitle = {Computer Vision – ECCV 2024: 18th European Conference, Milan, Italy, September 29–October 4, 2024, Proceedings, Part IX},
pages = {39–57},
numpages = {19},
location = {Milan, Italy}
}

@InProceedings{Xia_2025_CVPR_FG2,
    author    = {Xia, Zimin and Alahi, Alexandre},
    title     = {FG{\textasciicircum}2: Fine-Grained Cross-View Localization by Fine-Grained Feature Matching},
    booktitle = {Proceedings of the Computer Vision and Pattern Recognition Conference (CVPR)},
    month     = {June},
    year      = {2025},
    pages     = {6362-6372}
}

@inproceedings{
loshchilov2018decoupled,
title={Decoupled Weight Decay Regularization},
author={Ilya Loshchilov and Frank Hutter},
booktitle={International Conference on Learning Representations},
year={2019},
url={https://openreview.net/forum?id=Bkg6RiCqY7},
}

@article{paszke2019pytorch,
  title={PyTorch: An Imperative Style, High-Performance Deep Learning Library},
  author={Paszke, Adam and Gross, Sam and Massa, Francisco and Lerer, Adam and Bradbury, James and Chanan, Gregory and Killeen, Trevor and Lin, Zeming and Gimelshein, Naresh and Antiga, Luca and others},
  journal={Advances in Neural Information Processing Systems},
  volume={32},
  year={2019}
}

@inproceedings{xia2022visual,
  title={Visual Cross-View Metric Localization with Dense Uncertainty Estimates},
  author={Xia, Zimin and Booij, Olaf and Manfredi, Marco and Kooij, Julian FP},
  booktitle={European Conference on Computer Vision},
  pages={90--106},
  year={2022},
  organization={Springer}
}

\ifincludesuppl
  \clearpage
  \appendix

\section*{Appendix}
In this appendix, we provide the following:

\begin{itemize}
\item{Proof of Proposition~1 (Sec.~\ref{sec:proof_radial})}
\item{Training setup (Sec.~\ref{sec:training_setup})}
\item{Pseudo-code for Pair-wise Yaw Voting (Sec.~\ref{sec:pseudocode_dist})}
\item{Implementation Details (Sec.~\ref{sec:implementation_detail})}
\item{Baseline Evaluation Details (Sec.~\ref{sec:baseline_impl})}
\item{Preprocessing for Mapillary Geo-Localization Dataset (Sec.~\ref{sec:cross_view_mgl})}
\item{Visualization of Matching Process (Sec.~\ref{sec:visualization_supple})}
\item{Computational Cost Analysis (Sec.~\ref{sec:computation_cost_detail})}
\item{Failure Case Analysis (Sec.~\ref{sec:failure_cases})}
\item{Extended Yaw Analysis (Sec.~\ref{sec:extended_analysis})}
\item{Discussion and Future Work (Sec.~\ref{sec:discussion})}
\end{itemize}

\section{Proof of Proposition~1}
\label{sec:proof_radial}
\begin{proof}
If $(x_1, y_1)$, $(x_2, y_2)$, and $(u,v)$ are collinear, then $\tfrac{v - y_1}{u - x_1} = \tfrac{v - y_2}{u - x_2}$, so $\arctan\!\big(\tfrac{v-y_1}{u-x_1}\big) = \arctan\!\big(\tfrac{v-y_2}{u-x_2}\big)$. Since $\mathbf{Y}(k)$ depends only on the matched column, $\theta_{\mathrm{est}}$ is constant along the radial line through $(u,v)$.
\end{proof}

\section{Training Setup}
\label{sec:training_setup}
For the training setup, we based our implementation on the codebase provided by Shi et al.~\citet{Shi_2023Boosting}. We use NVIDIA RTX A6000 for training, with batch size $6$ and learning rate $10^{-4}$, using AdamW~\citep{loshchilov2018decoupled} optimizer. Loss weight $e^{-5}$ is used to adjust the loss scale. StepLR from PyTorch~\citep{paszke2019pytorch} is used, with a gamma of $0.25$ and step size of $4000$ iterations.

\section{Pseudo-code for Pair-wise Yaw Voting}
\begin{algorithm}
\caption{Pair-wise Yaw Voting with Distance Constraint}
\label{alg:pseudocode_dist}
\begin{algorithmic}[1]
  \State \textbf{Input:} Match yaw array $\mathbf{M_r}$, Match score array $\mathbf{M_s}$, Pose plausibility $\mathbf{S_P}\in\mathbb{R}^{|\Theta|\times(|\mathcal{X}|\times|\mathcal{Y}|)}$
  \State \textbf{Output:} Yaw bins for 2D poses $\mathbf{S}\in\mathbb{R}^{|\Theta|\times(|\mathcal{X}|\times|\mathcal{Y}|)}$
  \State $\mathbf{S}\gets\text{Zero-filled array}\in\mathbb{R}^{|\Theta|\times(|\mathcal{X}|\times|\mathcal{Y}|)}$
  \State $\theta_{\max}\gets\text{Max yaw noise (degrees)}$
  \State $R \gets \text{Max distance radius}$

  \ForAll{each candidate pose $(x,y)\in\mathcal{X}\times\mathcal{Y}$}
    \ForAll{each BEV pixel $(u,v)$ in All BEV pixels}
      \State $d \gets \sqrt{(u-x)^2 + (v-y)^2}$ \Comment{Calculate distance}
      \If{$d < R$} \Comment{Apply distance constraint}
        \State $\theta_{\text{abs}}\gets\arctan\!\left(\tfrac{v-y}{u-x}\right)$ \Comment{Absolute angle}
        \State $\theta_{\text{rel}}\gets\mathbf{M_r}(u,v)$ \Comment{Relative yaw}
        \State $\theta_{\text{est}}\gets\theta_{\text{abs}}-\theta_{\text{rel}}$ \Comment{Estimated yaw}
        \State $B_{\text{est}}\gets\lfloor\theta_{\text{est}}+0.5\rfloor$ \Comment{Yaw bin index}
        \If{$-\theta_{\max}\le B_{\text{est}}\le\theta_{\max}$}
          \State $\mathbf{S}(B_{\text{est}}, (x, y))\mathrel{+}= \mathbf{M_s}(u,v)$ \Comment{Vote match score to yaw}
        \EndIf
      \EndIf
    \EndFor
  \EndFor

  \State $\mathbf{S}\gets\mathbf{S}+\mathbf{S_P}$ \Comment{Pose plausibility adjustment}
\end{algorithmic}
\end{algorithm}
\label{sec:pseudocode_dist}
We provide the pseudo-code of the Yaw-aligned Bin Score Scattering method's key component, score accumulation in Alg.~\ref{alg:pseudocode_dist}, for clarification of its unique formulation. We wrote it in a loop form to make it more intuitive. PyTorch functions were used for effective implementation and GPU-accelerated parallelism. Further implementation details are in Sec.~\ref{sec:implementation_detail}.

The $\mathbf{S}$ array is indexed with coordinates and discrete yaws in the pseudo-code. $\mathcal{X}$ and $\mathcal{Y}$ are all valid $x$ and $y$ 2D camera pose's pixels in BEV space within noise range, and $\Theta$ represents valid discrete yaws in degree $[-\theta_{max}, -\theta_{max}+1, ..., \theta_{max}-1, \theta_{max}]$, which depends on the maximum yaw noise $\theta_{max}$.

\section{Implementation Details}
\label{sec:implementation_detail}
This section provides a detailed formulation of the method not discussed in the main paper for brevity. The variables denoted using the notation of a hat, such as $\hat{x}$ instead of $x$, are used for implementation.

\subsection{Feature Extraction}
We utilize the implementation of U-Net~\citep{unet} built on the VGG-16~\citep{VGG-16} encoder backbone, which features separate weights for the ground and the BEV, as provided in the official code of Shi et al.~\citet{Shi_2023Boosting}, along with an additional heading encoding channel. Instead of the cosine-based representation used by Wang et al.~\citet{Wang_2023PureACL}, we use linear mapping of the $x$ coordinate from $-1$ to $1$ and concatenate them along the channel axis.

Unlike the original implementation of Shi et al.~\citet{Shi_2023Boosting}, the features are not normalized per image. Instead, ground and BEV features are normalized along channel dimensions:
\begin{equation}
\mathbf{\hat{F}}_\text{bev}(u, v) = \frac{\mathbf{F}_\text{bev}(u, v)}{\|\mathbf{F}_\text{bev}(u, v)\|_2}, \quad \mathbf{\hat{F}}_\text{grd}(c) = \frac{\mathbf{F}_\text{grd}(c)}{\|\mathbf{F}_\text{grd}(c)\|_2}
\end{equation}
$\mathbf{\hat{F}}_\text{bev}$ and $\mathbf{\hat{F}}_\text{grd}$ are used for the cosine-similarity computation explained in the method section.

The confidence for each feature is computed using the original feature before normalization. For the ground confidence estimator, output values are directly from the dense layer, while the BEV confidence estimator values are the sigmoid of the dense layer output. Max normalization is applied to the BEV confidence value, so the maximum confidence is 1.
\begin{equation}
\mathbf{C}_{\text{bev}}^{\text{max}} = \max_{u, v} \mathbf{C}_{\text{bev}}(u, v)
\end{equation}
\begin{equation}
\mathbf{\hat{C}}_{\text{bev}}(u, v) = \frac{\mathbf{C}_{\text{bev}}(u, v)}{\mathbf{C}_{\text{bev}}^{\text{max}}}
\end{equation}

\subsection{Pair-wise Yaw Voting with Distance Constraint}
We only consider possible positions under the assumption of the noise range for the score bins. The score bin is $\mathbb{R}^{|\Theta| \times (|\mathcal{X}| \times |\mathcal{Y}|)}$. \(\mathcal{X}\) and \(\mathcal{Y}\) are the sets of ranges for the $x$ and $y$ of pixels covering the maximal location noise range:
\begin{align}
\mathcal{X} &= \left\{ \left[i \Delta x, \, (i+1) \Delta x\right) \,\middle|\, i \in \{-L, \dots, L-1\} \right\} \\
\mathcal{Y} &= \left\{ \left[j \Delta y, \, (j+1) \Delta y\right) \,\middle|\, j \in \{-L, \dots, L-1\} \right\}
\end{align}
\noindent where $\Delta x$ and $\Delta y$ correspond to the same meter per pixel value of the BEV feature. Denoting maximal location noise along one direction as $N_\textbf{max}$, $L$ is:
\begin{equation}
L = \left\lceil \frac{N_\textbf{max}}{\Delta x} \right\rceil
\end{equation}

\(\Theta\) is the range of \(\theta\) in degrees, from the ceiling of $-\theta_{\textbf{max}}$ to $\theta_{\textbf{max}}$ divided by each resolution's bin's angular interval $\Delta \theta$ explained in the Sec.~\ref{sec:multi_res}.
\begin{equation}
\Theta_i = \left\{ \left ((i-1) \Delta \theta, i \Delta \theta \right), i \in \{-n, \dots, -1, 0, 1, \dots, n-1\} \right \}
\end{equation}
\noindent where
\begin{equation}
n = \left\lceil \frac{\theta_{\text{max}}}{\Delta \theta} \right\rceil
\end{equation}

We choose a fixed $R = 36.32$, subtracting $20$m from the maximal distance from the center $256 \times 0.22$m, for pixel length and meters per pixel for the Mapillary Geo-localization dataset~\citep{sarlin2023orienternet} and the Ford dataset~\citep{Ford-Multi-AV}. For the VIGOR dataset~\citep{zhu2021vigor}, we use a quarter of BEV's length of one side for $R$, as the meter per pixel varies across different cities.

\subsection{Pose Plausibility Score Model}
The pose plausibility score $\mathbf{S_P}$ is computed with a model based on Shi et al.~\citep{Shi_2023Boosting}'s Uncertainty estimator implementation. BEV features from a separate VGG-16 U-Net are used for input to the model, as sharing features disturbs the training of the main match scores. The intermediate layer dimension is adjusted to fit the output channel and number of angular bins for each resolution. The LeakyReLU~\citep{xu2015empiricalevaluationrectifiedactivations} replaces the final non-linearity with a negative slope $0.01$.

\subsection{Multi Resolution Scoring}
\label{sec:multi_res}
The final score is computed for feature at each resolution, following common practice~\cite{Shi_2023Boosting, Wang_2023PureACL, Wang_2024OVCL}. The ground and BEV image features with the same downsampling ratio from the U-Net are used for each resolution. Following Shi et al.~\citep{Shi_2023Boosting}, we use features downsampled 2, 4, and 8 times for scoring, and the score from the 2 times downsampling features is used for the final estimation. Each resolution is supervised using the log-probability loss described in the paper, with the lower-resolution feature used for coarser estimation.

The score bin is built per pixel for different coarseness levels for each resolution. The angular bin has discrete intervals of 1, 2, and 4 degrees for each resolution, with larger intervals for lower resolutions. The score is repeated along each axis and summed to a higher resolution. The summation is done before applying kernels for the scattering-based match score $S_M$ and the 3D plausibility score $S_P$.

\section{Preprocessing for Mapillary Geo-Localization Dataset}
\label{sec:cross_view_mgl}
The dataset, paired with a 2D map, can be downloaded by following instructions from OrienterNet~\citep{sarlin2023orienternet}'s repository. The data contains ground-view images with camera pose information. We use the urban Amsterdam part of the dataset for the cross-view setup. The Google Static Maps API fetches satellite-view images, and we filtered out top-down occluded views. Consistent with previous work~\citep{Shi_2023Boosting, Wang_2024OVCL}, satellite views are obtained as $1280\times1280$ images with a zoom level of 18 and a scale of 2. The satellite view is fetched for each ground view; however, no new image is downloaded if any existing satellite image covers a $512\times512$ region around the ground view.

OrienterNet~\citep{sarlin2023orienternet}, which introduced the Mapillary Geo-Localization Dataset, states that they removed ground views with low visibility of the surroundings, such as those facing a wall. We found such images remained, and removed those portions. The Amsterdam dataset is generally at a $512\times512$ resolution; we filtered out a small number of images with different resolutions to ensure a fair evaluation against the baselines. The train and test sets are split by sequence index to minimize spatial overlap.

\section{Baseline Evaluation Details}
\label{sec:baseline_impl}
\begin{table}[!ht]
\centering
\caption{\textbf{Summary of baseline official implementations for Ford, KITTI, and VIGOR datasets.} For each implementation status, we did the following: \textbf{None}: Dataloader adapted from external codebase. \textbf{Metrics on Paper}: Results cited from the original paper. \textbf{Small Yaw Range}: Retrained for high-yaw-noise. \textbf{Unused in Paper}: The official code's unused train script is fixed to evaluate full 3-DoF pose estimation. \textbf{Only Location}: The official code trains and evaluates only the location estimator for the dataset; the rotation estimator training code is adapted from another dataset script.}
\label{tab:baseline_datasets}
\begin{tabular}{|l|c|c|c|}
\hline
Method & Ford & KITTI & VIGOR \\
\hline
CCVPE~\citep{xia2023_CCVPE} & None & Metrics on Paper & Metrics on Paper \\
BoostAcc~\citep{Shi_2023Boosting} & Small Yaw Range & Small Yaw Range & Unused in Paper \\
G2S~\citep{g2sweakly_2024_shi} & None & Small Yaw Range & Only Location \\
FG2~\citep{Xia_2025_CVPR_FG2} & None & Small Yaw Range & Metrics on Paper \\
\hline
\end{tabular}
\end{table}
In this section and Table~\ref{tab:baseline_datasets}, we describe how each baseline uses each dataset and how we adapted the official code to evaluate each method at its best under our setting.

\subsection{Per-Dataset Adaptation}
The \textbf{Ford dataset} has been evaluated less frequently in recent work; only BoostAcc~\cite{Shi_2023Boosting} has reported results on it. Since they train and evaluate on the $\pm10^\circ$ yaw noise, we retrained the method with a $\pm45^\circ$ yaw range. The BoostAcc dataloader is adapted for both our method and other baselines for training and evaluation.

The \textbf{KITTI dataset} is widely used as a benchmark, yet methods commonly evaluate with only $\pm10^\circ$ yaw error. We used official codes and modified the rotation range during training and testing to evaluate baselines. CCVPE~\cite{xia2023_CCVPE} reported an official result with unknown yaw, but the weights were not released. CCVPE reports only $1^\circ/3^\circ/5^\circ$ accuracies, which we place in our $1^\circ/2^\circ/4^\circ$ columns. Since accuracy increases with the threshold, its $2^\circ/4^\circ$ entries are higher, conservative than reality.

For the \textbf{VIGOR dataset}, CCVPE~\cite{xia2023_CCVPE} and FG2~\cite{Xia_2025_CVPR_FG2} report unknown yaw estimation accuracy, and released model source code. FG2 released the VIGOR weight, and we ran inference to compute the accuracy for our angle threshold. CCVPE did not release the VIGOR weight, so we use the metrics reported in the paper. BoostAcc~\cite{Shi_2023Boosting} did not use the dataset in the paper, but has a dataloader in the released code; we use it to retrain and test the model. G2S~\cite{g2sweakly_2024_shi} provides a dataloader, but they did not use their rotation estimator for unknown yaw setups, stating that panoramic ground-view and BEV differences make it difficult to apply their self-supervised method for yaw estimation.

For the \textbf{MGL dataset}, none of the methods support it; thus, we adapted the official dataloader from OrienterNet~\cite{sarlin2023orienternet} code to use BEV input, for all baselines and ours.

\subsection{Behavior Across Yaw Uncertainty Ranges}
\label{sec:degradation_range}
\ourmethod{}'s margin over the baselines in Tables~\ref{tab:comparison_kitti} and~\ref{tab:comparison_vigor}\ofmainpaper{} widens as the yaw candidate range grows.
Table~\ref{tab:comparison_mgl}\ofmainpaper{} shows this on MGL across $\pm45^{\circ}$ and $\pm180^{\circ}$ noise (a $4{\times}$ expansion of the candidate range): CCVPE's $<1^{\circ}$ accuracy drops from $21.36\%$ to $6.55\%$, FG2 from $18.59\%$ to $3.13\%$, BoostAcc from $7.40\%$ to $0.58\%$, and G2S from $32.67\%$ to $0.59\%$.
Refinement-based methods (G2S, BoostAcc) fall the most: with a limited-FoV camera, their refinement cannot escape an initial $\sim180^{\circ}$ error. \ourmethod{} also drops ($72.10 \rightarrow 34.81$), but far less.

\section{Visualization of Matching Process}
\label{sec:visualization_supple}

\begin{figure}
    \includegraphics[width=0.98\linewidth]{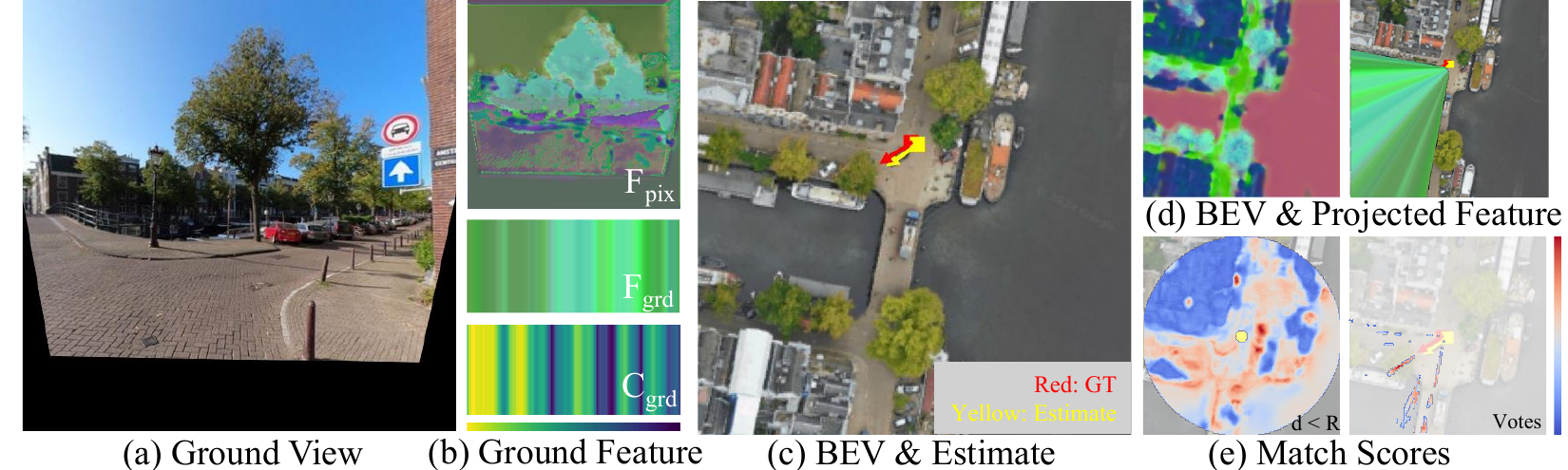}
    \caption{Visualization of our yaw estimation in the MGL dataset. Our method gives high confidence and a match score for the column-aligned road on the left and the tree in the middle.}
    \label{fig:appendix_example1}
\end{figure}
\begin{figure}
    \includegraphics[width=0.98\linewidth]{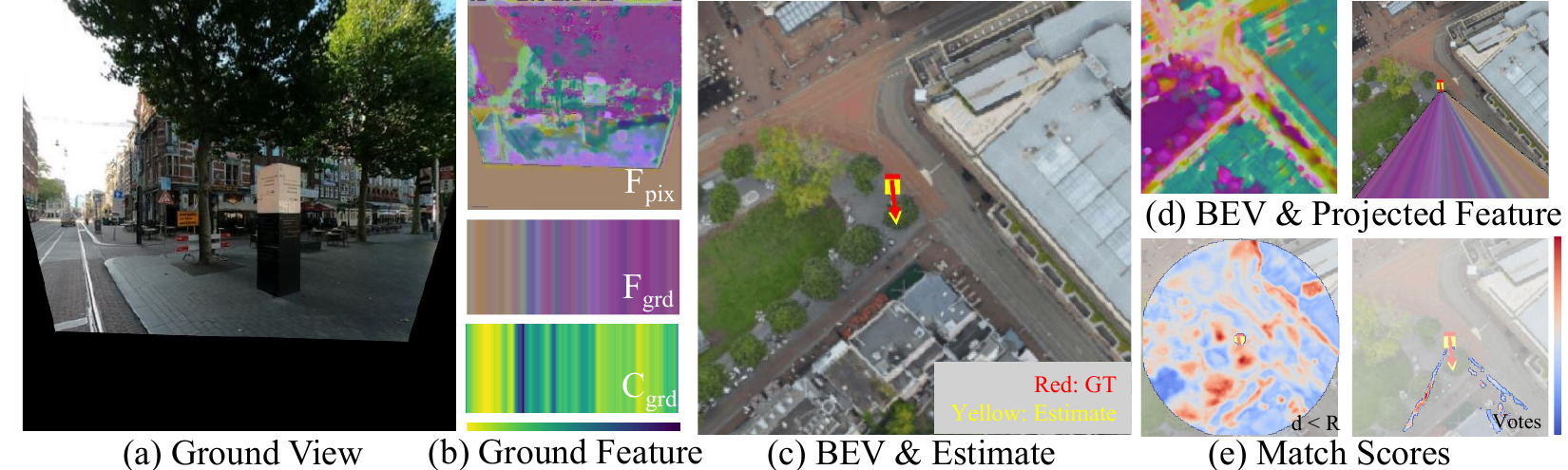}
    \caption{Visualization of our yaw estimation in the MGL dataset. The estimated pose has the most votes from a tree.}
    \label{fig:appendix_example2}
\end{figure}
\begin{figure}
    \includegraphics[width=0.98\linewidth]{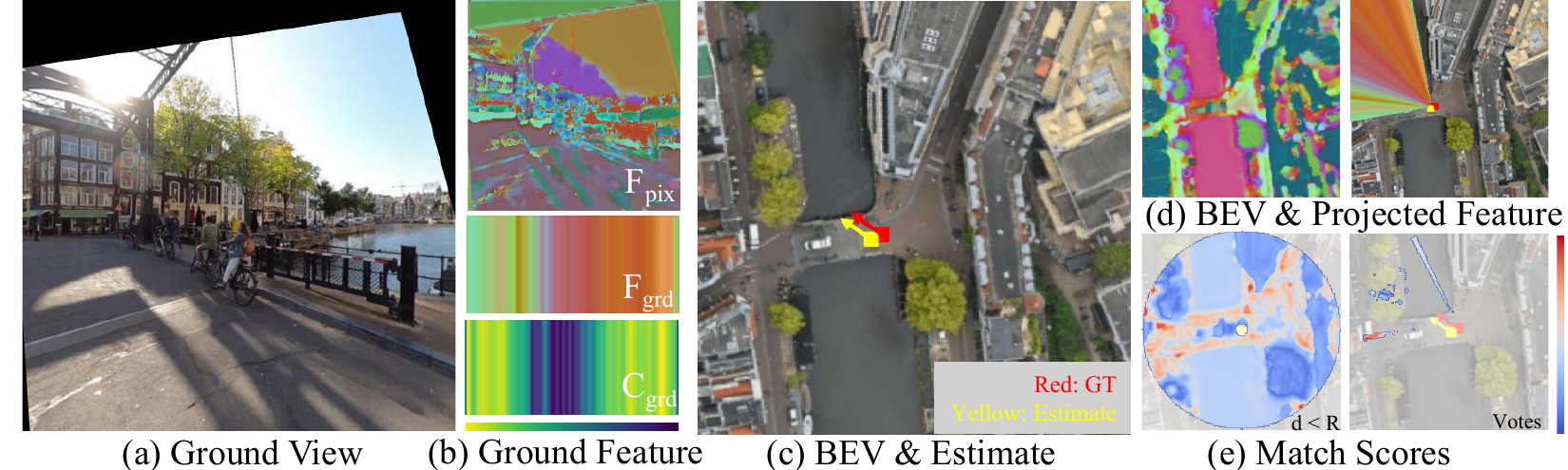}
    \caption{Visualization of our yaw estimation in the MGL dataset. Our method has high confidence and a match score on the features extracted from the building at the front of the road.}
    \label{fig:appendix_example3}
\end{figure}

\begin{figure}
    \includegraphics[width=0.95\linewidth]{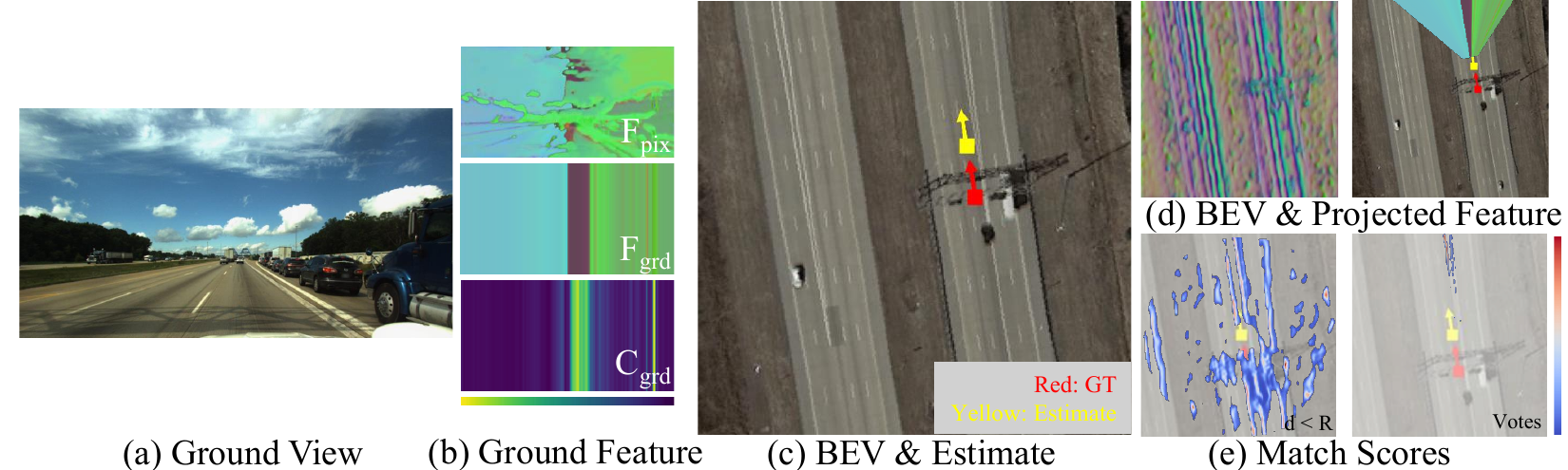}
    \caption{Visualization of our yaw estimation in the Ford dataset. Our method aligns the road components.}
    \label{fig:appendix_example_ford}
\end{figure}

\begin{figure}
    \includegraphics[width=0.95\linewidth]{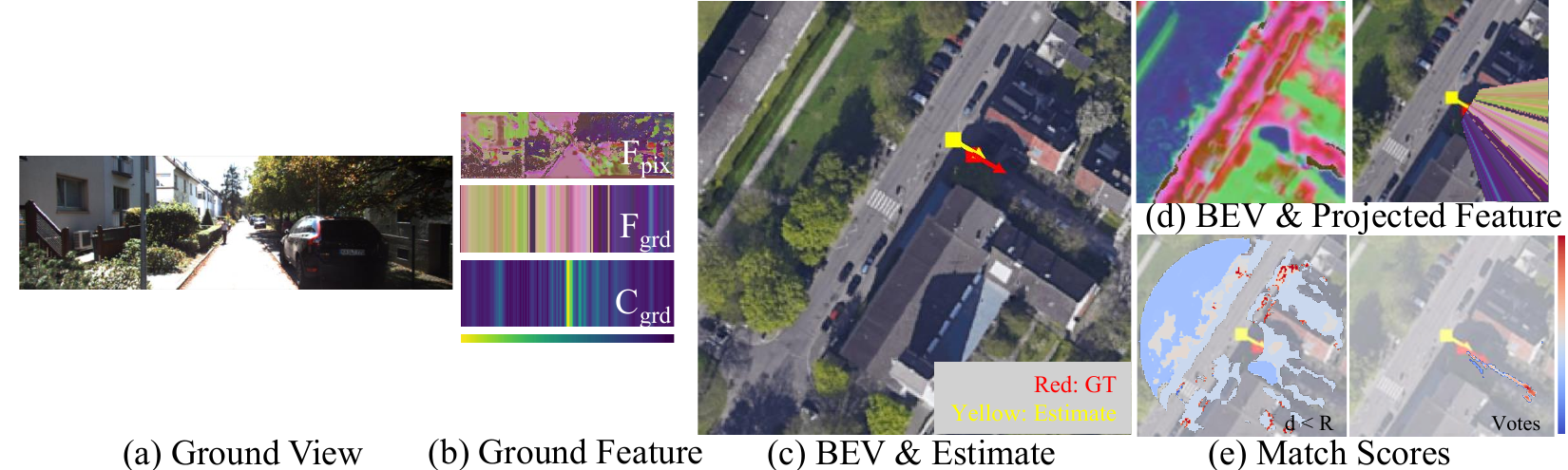}
    \caption{Visualization of our yaw estimation in the KITTI dataset, same-area (Test~1). Despite the wide-baseline perspective change between the ground view and BEV, our method identifies road and roadside structure correspondences to estimate yaw.}
    \label{fig:appendix_example_kitti_same}
\end{figure}

\begin{figure}
    \includegraphics[width=0.95\linewidth]{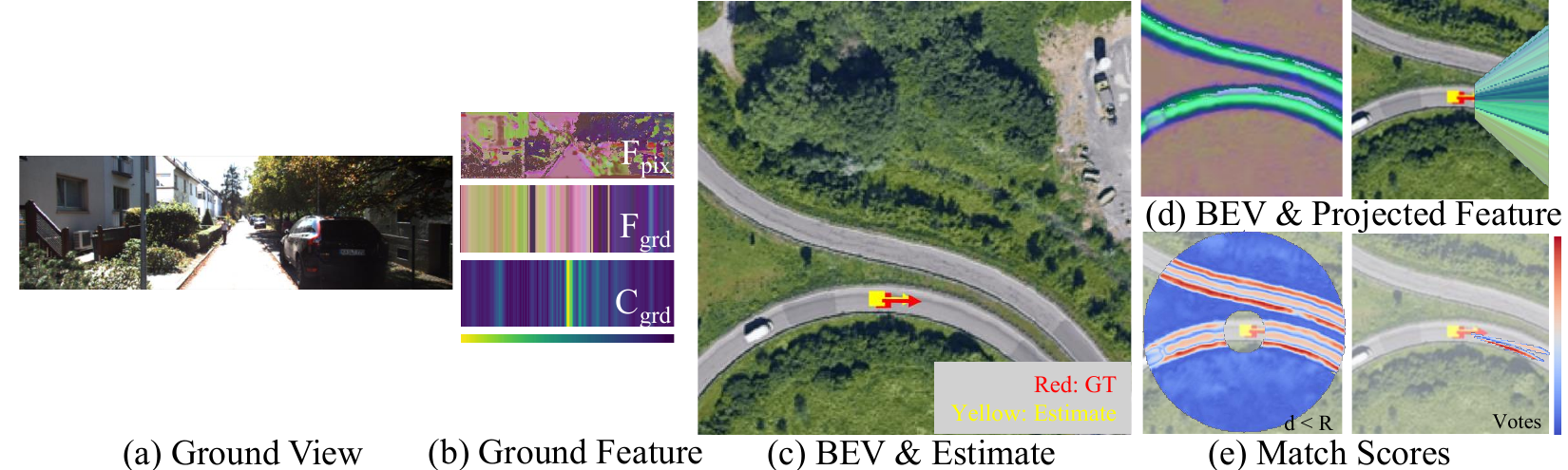}
    \caption{Visualization of our yaw estimation in the KITTI dataset, cross-area (Test~2). Our method generalizes to unseen geographic areas. The same area is mostly urban; our method generalizes to curved highways in the cross-area data, finding radial correspondences along road boundaries.}
    \label{fig:appendix_example_kitti_cross}
\end{figure}

\begin{figure}
    \includegraphics[width=0.95\linewidth]{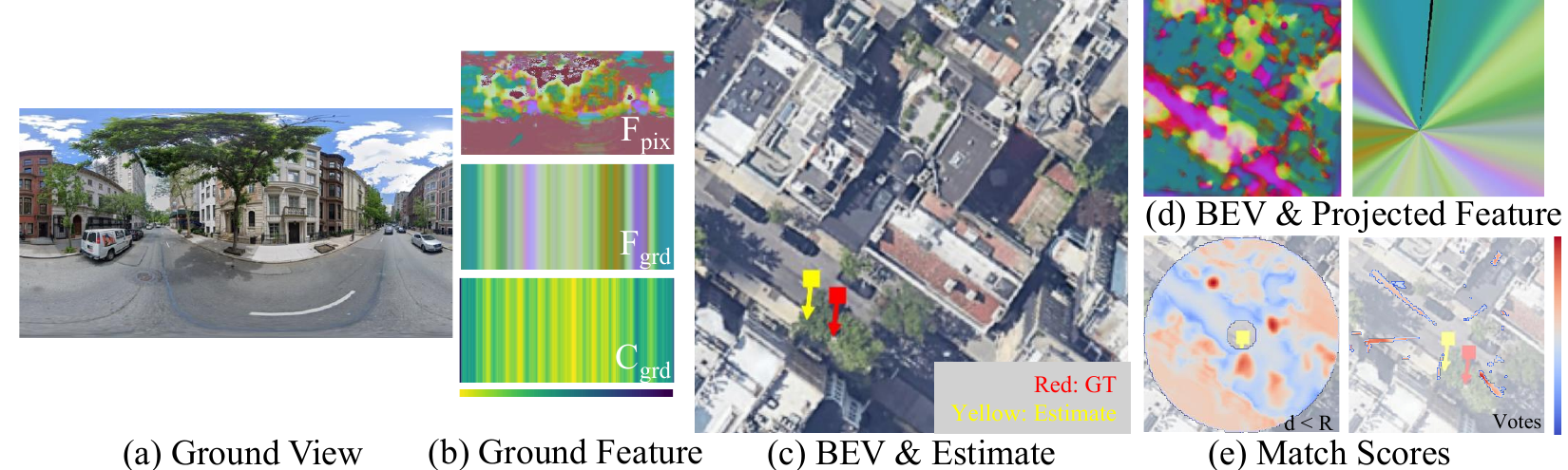}
    \caption{Visualization of our yaw estimation in the VIGOR dataset, same-area setup. Our method aligns a specific building and road.}
    \label{fig:appendix_example_vigor_same}
\end{figure}

\begin{figure}[t]
    \includegraphics[width=0.95\linewidth]{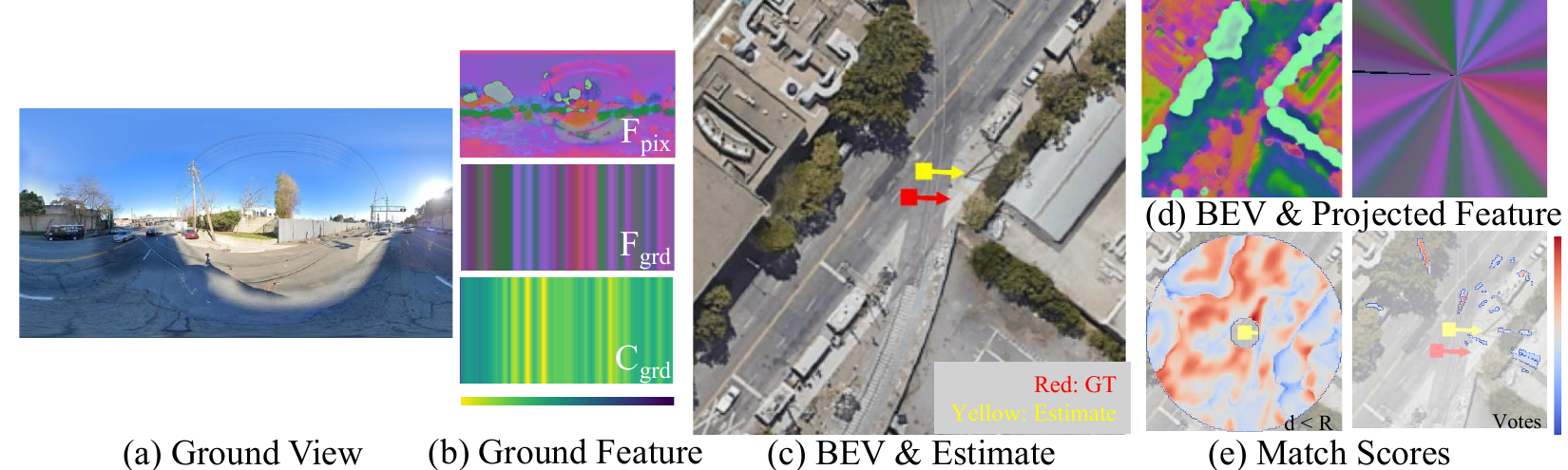}
    \caption{Visualization of our yaw estimation in the VIGOR dataset, cross-area setup. Our method aligns the road and a tree.}
    \label{fig:appendix_example_vigor_cross}
\end{figure}

In this section, we demonstrate the matching process and explain how our method accurately estimates yaw with more examples. We visualize the matching process using Principal Component Analysis (PCA). We extract three major components to represent it as RGB channels. For each sample, ground column features and BEV features are sampled and re-scaled for PCA. Since we use an absolute similarity score, feature vectors are normalized by flipping the sign when they have a negative similarity with the mean of all feature vectors.

Experiments in the MGL~\citep{sarlin2023orienternet} dataset revealed that our method not only focuses on the road, but also on prominent roadside features. Fig.~\ref{fig:appendix_example1} and Fig.~\ref{fig:appendix_example2} provide examples from the MGL dataset that focus on a tree, colored in dark purple and cyan, respectively. Fig.~\ref{fig:appendix_example3} shows the case where our process focuses on the building in the road front direction, with the column feature visualized as light brown on the left side of the ground image. Generally, our method estimates yaw based on the focused score in one direction. 

The Ford~\citep{Ford-Multi-AV} and VIGOR~\citep{zhu2021vigor} dataset results primarily focus on the driveway where the datasets were captured. An example from the Ford dataset is visualized in Fig.~\ref{fig:appendix_example_ford}, with ground-view-related features resized to their original aspect ratio. The features from the left and right sides of the view appear different, due to positional encoding attached to the input image. Match scores are focused on road lanes, and the final estimation primarily comes from the lane in the front direction.

The KITTI~\citep{Geiger2013KITTI} dataset results are shown in Fig.~\ref{fig:appendix_example_kitti_same} (same-area, Test~1) and Fig.~\ref{fig:appendix_example_kitti_cross} (cross-area, Test~2). The cross-area example demonstrates generalization to previously unseen geographic regions. Similar to the Ford dataset, matches are focused on road structure, as KITTI also focuses on ground-level views from a moving car.

Fig.~\ref{fig:appendix_example_vigor_same} shows the matching in the VIGOR dataset, same-area setup. Our method identifies correspondences such as sidewalk boundary (brown) and driveway (purple).
Fig.~\ref{fig:appendix_example_vigor_cross} illustrates the matching process using the VIGOR dataset, specifically the cross-area setup. At the road intersection, our method identifies correct road correspondence and finds yaw, focusing on a road (purple) and a tree (dark green) near the other road.

\section{Computational Cost Analysis}
\label{sec:computation_cost_detail}
\begin{table}[t]
\centering
\caption{Per-batch forward pass breakdown}
\label{tab:computational_cost}
\begin{tabular}{|l|c|}
\hline
Stage & Percentage in Time \\
\hline
Feature extraction & 24.2 \\
Pose plausibility network & 0.9 \\
Ground-BEV matching & 1.7 \\
Yaw voting & 73.2 \\
\hline
\end{tabular}
\end{table}

Table~\ref{tab:computational_cost} breaks down the per-batch forward pass by stage. The majority of computation comes from voting; however, it remains comparable to the feature extraction process. Note that the voting with absolute and relative yaw indexing is not implemented efficiently using the existing PyTorch function. A better kernel implementation can improve runtime performance.

\section{Failure Case Analysis}
\label{sec:failure_cases}

We identify three main categories of failure cases:

\paragraph{Severe Temporal Change}
When the BEV image and ground image are captured at significantly different times, scene content may have changed substantially due to construction, seasonal vegetation changes, or demolition. In such cases, the learned features may fail to find correct correspondences because the visual content no longer matches between views. Datasets include some temporal variation, and our method demonstrates robustness to moderate changes; however, extreme differences (e.g., a demolished building or newly constructed road) can cause failures.

We observed that other failures often accompanied more common temporal changes, such as mismatched weather conditions, day and night, and trees blocking the road. While those typically do not significantly impact estimation, since those shifts can be learned during training, they can impair performance when combined with other sources of confusion, such as the failure cases below.

\begin{figure}[!ht]
    \includegraphics[width=0.95\linewidth]{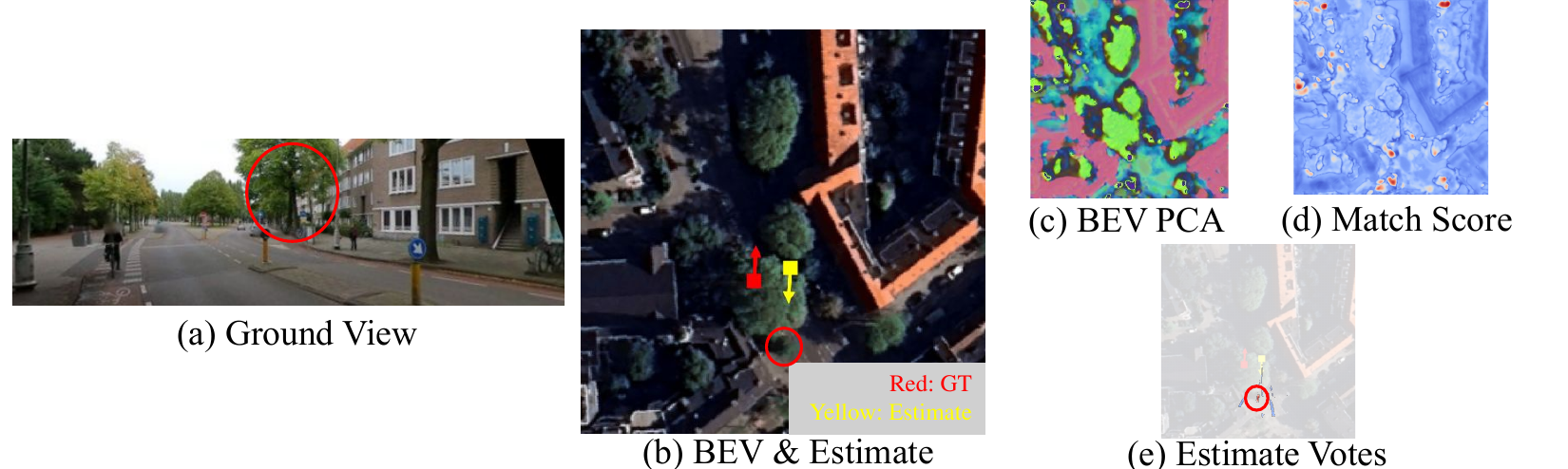}
    \caption{Failure case examples in the MGL dataset \textbf{Severe Occlusion of BEV}: road is completely covered by trees when the BEV image was captured. Severe occlusion makes it difficult to match ground view and BEV, especially in day night time change from ground view to BEV. For the estimated pose, vote mainly comes from a small tree in the BEV, and high match scores (red color) are generally on small trees, which aligns with the tree in the middle of image.}
    \label{fig:failure_cases_occlusion}
\end{figure}

\paragraph{Severe Occlusion of BEV}
There are some ground-level images taken in places that are occluded from a top-down view, such as a road completely covered by trees in Fig.~\ref{fig:failure_cases_occlusion}. Due to a lack of matching visual features, our method focuses on a small tree in the ground view; however, it matches a different tree, and without support from surrounding visuals, this leads to yaw estimation failure.

\begin{figure}[!ht]
    \includegraphics[width=0.95\linewidth]{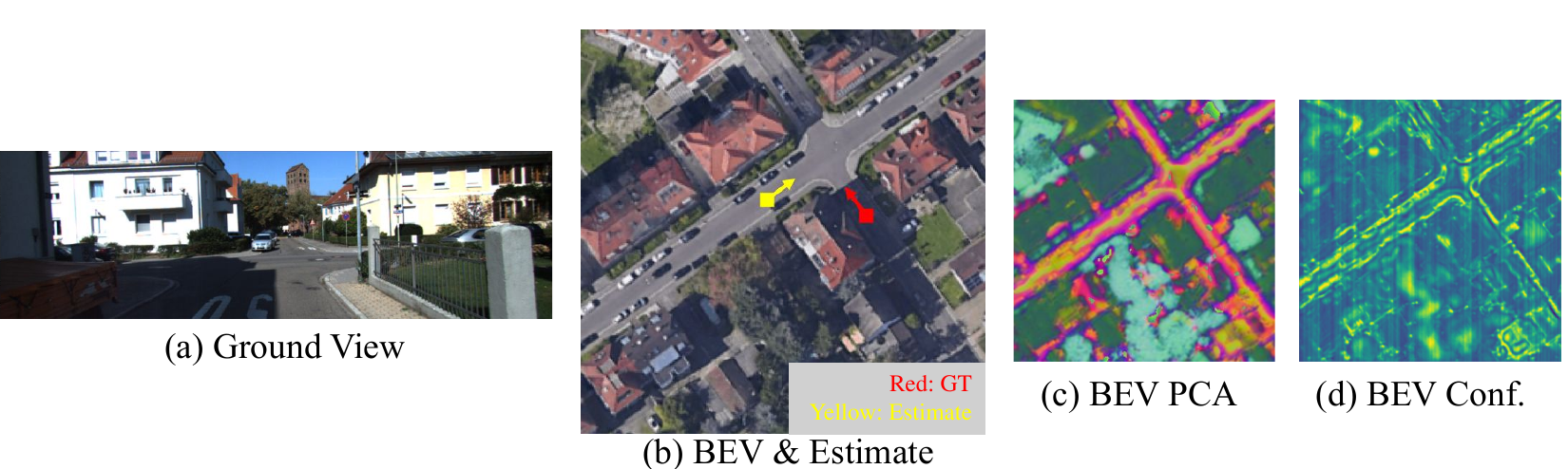}
    \caption{Failure case examples in the KITTI dataset. \textbf{Symmetric scene}: a crossroad produces ambiguous votes between two visually similar directions. Temporal impact, such as shades on the road, reduced confidence (viridis colormap, yellow indicates high confidence, green indicates lower) on the correct road, resulting in similar-looking different road positions selected.}
    \label{fig:failure_cases_symmetry}
\end{figure}

\paragraph{Symmetric Scenes}
At locations with rotational symmetry, such as crossroads with four visually similar streets, circular plazas, or repetitive building patterns, multiple radial directions may produce equally strong votes, such as Fig.~\ref{fig:failure_cases_symmetry}, leading to ambiguous yaw estimates. The voting mechanism may split its consensus across multiple yaw angles or select an incorrect symmetric alternative. This failure mode is inherent to any appearance-based matching method operating under high yaw uncertainty.

\section{Extended Analysis}
\label{sec:extended_analysis}
\subsection{Yaw Error Distribution}
\label{sec:yaw_full_stats}
Table~\ref{tab:yaw_threshold_stats} reports the 1/2/3/4/5-degree threshold accuracies together with the median and mean yaw error for MGL and KITTI, where threshold metrics are the standard summary for limited-FoV camera datasets. The $1/2/4^{\circ}$ values differ slightly from the main table, recomputed after a later code update that left the method core unchanged. The same pattern holds on VIGOR, where the median yaw error is $1.94^{\circ}$ (same-area) and $5.45^{\circ}$ (cross-area), far below the mean of $35.7^{\circ}$ and $45.5^{\circ}$. The gap stems from a front/back ambiguity: a column-aggregated road correspondence looks nearly identical $180^{\circ}$ apart, so a small fraction of predictions flip and inflate the mean. Asymmetric cues such as building facades, or a temporal prior (speed, heading continuity) in navigation, resolve this binary ambiguity, making the median and threshold accuracy the representative indicators.

\begin{table}[!ht]
\centering
\caption{Yaw threshold accuracy (\%, $\uparrow$) and median/mean yaw error (\textdegree, $\downarrow$) for \ourmethod{} on MGL and KITTI under $\pm180^{\circ}$ yaw uncertainty.}
\label{tab:yaw_threshold_stats}
\setlength{\tabcolsep}{5pt}
\scriptsize
\begin{tabular}{|l|ccccc|cc|}
\hline
\multirow{2}{*}{Split} & \multicolumn{5}{c|}{Threshold accuracy (\%) $\uparrow$} & \multicolumn{2}{c|}{Error (\textdegree) $\downarrow$} \\
 & $<1^{\circ}$ & $<2^{\circ}$ & $<3^{\circ}$ & $<4^{\circ}$ & $<5^{\circ}$ & Median & Mean \\
\hline
MGL                  & 34.81 & 52.42 & 58.78 & 61.74 & 63.29 & 1.79 & 44.26 \\
KITTI (same-area)    & 51.07 & 62.47 & 64.06 & 64.64 & 64.75 & 0.96 & 46.17 \\
KITTI (cross-area)   & 48.44 & 57.62 & 58.63 & 58.80 & 58.87 & 1.07 & 54.94 \\
\hline
\end{tabular}
\end{table}

\subsection{Matching-only Estimation}
\label{sec:sm_only}
Our framework combines $S_M$ with a learned pose-plausibility term $S_P$. To isolate the yaw signal carried by $S_M$, we set $S_P$ to zero at every voting level of a trained \ourmethod{} model and re-run the forward pass.

On MGL dataset, the matching-only score at the GT location ranks the GT yaw bin in the top $2.5\%$ of all yaw bins (mean rank percentile $0.926$, median $0.975$).
The matching-only sub-degree accuracy at the GT location is $15.10\%$, $<2^{\circ}$ is $28.47\%$, $<4^{\circ}$ is $49.07\%$, above the strongest baselines in Table~\ref{tab:comparison_mgl}\ofmainpaper{}.

Training \ourmethod{} with $S_P$ removed converges but is unstable: early matching scores are near-uniform across many 2D pose bins, so the softmax loss is dominated by correct-match-to-wrong-location votes before features become discriminative, reaching $14.18 / 25.84 / 44.25$ within $1^{\circ}/2^{\circ}/4^{\circ}$ on MGL $\pm180^{\circ}$. $S_P$ therefore serves mainly as an early-training stabilizer, while radial voting carries the yaw signal.

\section{Discussion and Future Work}
\label{sec:discussion}
Our method has limitations in exact location estimation when used standalone, due to the formulation that adds a match score to all pixels. Combining line-aligning scoring with the conventional projection-based method has the potential to yield more accurate yaw and location measurements.

Further improving cross-view localization for outdoor AR and MR applications is a promising direction; however, it poses new challenges due to the significantly greater pose freedom and the rich surrounding environment. Most existing vehicle-based works assume that features are mapped to flat ground and roads in the vehicle dataset. While the MGL dataset provides a more varied environment, the less-structured scenes typical of AR and MR applications present significant opportunities for future exploration.

The 2D map-based localization~\citep{sarlin2023orienternet, 2024maplocnet} and the cross-view localization have different advantages. The 2D map provides semantic information about the surroundings but does not provide exact visual cues. The BEV from the cross-view localization setup provides precise visual cues but lacks semantic meaning. The fusion of two modalities would guide a new direction to more challenging AR and MR localization.

\fi
\end{document}